\documentclass[10pt,twocolumn,letterpaper]{article}

\usepackage{cvpr}
\usepackage{times}
\usepackage{epsfig}
\usepackage{graphicx}
\usepackage{amsmath}
\usepackage{amssymb}
\usepackage[pagebackref=true,breaklinks=true,letterpaper=true,colorlinks,bookmarks=false]{hyperref}

\usepackage{graphicx, color, xspace}
\usepackage{caption}
\usepackage{subcaption}
\usepackage{booktabs}

\def\real{{\cal R}}

\def\eg{\emph{e.g}.}

\newenvironment{proof}{\paragraph{Proof:}}{$\blacksquare$}
\DeclareMathOperator{\diag}{diag}
\DeclareMathOperator{\reluop}{ReLU}
\DeclareMathOperator{\vol}{vol}
\def\relu{\ensuremath \reluop}

\def\relus{{\ensuremath{\operatorname{\mathop{ReLU6}\,}}}}
\def\conv{{\ensuremath{\operatorname{\mathop{conv2d}\,}}}}
\def\depthwise{{\ensuremath{\operatorname{\mathop{dwise}\,}}}}

\usepackage{fullpage}
\newtheorem{theorem}{Theorem}
\newtheorem{lemma}{Lemma}

\newcommand{\oldtext}[1]{{\color{blue}{#1}}}
 \cvprfinalcopy %

\def\cvprPaperID{3427} %

\ifcvprfinal\fi

\begin{document}

\title{MobileNetV2: Inverted Residuals and Linear Bottlenecks}

\author{
\begin{tabular}{c c c c c}
Mark Sandler & Andrew Howard & Menglong
Zhu & Andrey Zhmoginov & Liang-Chieh Chen \\
\multicolumn{5}{c}{Google Inc.} \\
\multicolumn{5}{c}{\{sandler, howarda, menglong, azhmogin, lcchen\}@google.com} \\
\end{tabular}
}

\maketitle

\ifcvprfinal\thispagestyle{empty}\fi

\begin{abstract}
In this paper we describe a new mobile architecture, \mbox{MobileNetV2}, that improves the state of the art performance of mobile models on multiple tasks and benchmarks as well as across a spectrum of different model sizes.
We also describe efficient ways of applying these mobile models to object detection in a novel framework we call \mbox{SSDLite}.
Additionally, we demonstrate how to build mobile semantic segmentation models through a reduced form of \mbox{DeepLabv3} which we call Mobile \mbox{DeepLabv3}.

 is based on an inverted residual structure where the shortcut connections are between the thin bottleneck layers. The intermediate expansion layer uses lightweight depthwise convolutions to filter features as a source of non-linearity.  Additionally, we find that it is important to remove non-linearities in the narrow layers in order to maintain representational power. We demonstrate that this improves performance and provide an intuition that led to this design.

Finally, our approach allows decoupling of the input/output domains from the expressiveness of the transformation, which 
provides a convenient framework for further analysis.
We measure our performance on \mbox{ImageNet}~\cite{Russakovsky:2015:ILS:2846547.2846559} classification, COCO object detection \cite{COCO}, VOC image segmentation \cite{PASCAL}. We evaluate the trade-offs between accuracy, and number of operations measured by multiply-adds (MAdd), as well as actual latency, and the number of parameters. 
\end{abstract}

\section{Introduction}
Neural networks have revolutionized many areas of machine intelligence, enabling superhuman accuracy for challenging image recognition tasks. However, the drive to improve accuracy often comes at a cost: modern state of the art networks require high computational resources beyond the capabilities of many mobile and embedded applications. 

This paper introduces a new neural network architecture that is specifically tailored for mobile and resource constrained environments. Our network pushes the state of the art for mobile tailored computer vision models, by significantly decreasing the number of operations and memory needed while retaining the same accuracy.

Our main contribution is a novel layer module: the inverted residual with linear bottleneck. This module takes as an input a low-dimensional compressed representation which is first expanded to high dimension and filtered with a lightweight depthwise convolution. Features are subsequently projected back to a low-dimensional representation with a {\em linear convolution}. The official implementation is available as part of TensorFlow-Slim model library in~\cite{MobilenetV2-implementation}.

This module can be efficiently implemented using standard operations in any modern framework and allows
our models to beat state of the art along multiple performance points using standard benchmarks.
Furthermore, this convolutional module is particularly suitable for mobile designs, because it allows to significantly reduce the memory footprint needed during inference by never fully materializing large intermediate tensors. This reduces the need for main memory access in many embedded hardware designs, that provide small amounts of very fast software controlled cache memory. 

\section{Related Work}
Tuning deep neural architectures to strike an optimal balance between accuracy and performance has been an area of active research for the last several years.
Both manual architecture search and improvements in training algorithms, carried out by numerous teams has lead to dramatic improvements over early designs  such as \mbox{AlexNet}~\cite{AlexNet}, \mbox{VGGNet}~\cite{VGGNet}, \mbox{GoogLeNet}~\cite{GoogleNet}. , and \mbox{ResNet}~\cite{ResNet}.
Recently there has been lots of progress in algorithmic architecture exploration included hyper-parameter optimization \cite{BergstraRandomSearch, BayesianOptimizationML, BayesianOptimizationDNN} as well as various methods of network pruning \cite{BrainSurgeon,BrainDamage,Han2015,Han2016,Guo2016,LiPruning} and connectivity learning \cite{ConnectivityLearning,BudgetedSuperNetworks}.
A substantial amount of work has also been dedicated to changing the connectivity structure of the internal convolutional blocks such as in \mbox{ShuffleNet} \cite{ShuffleNet2017} or introducing sparsity \cite{PowerOfSparsity} and others \cite{SpatialBottleneck2016}. 

Recently, \cite{LearningToLearnScale,GeneticCNN,EvolutionImageClassifiers,NAS_reinforcement}, opened up a new direction of bringing optimization methods including genetic algorithms and reinforcement learning to architectural search.
However one drawback is that the resulting networks end up very complex.
In this paper, we pursue the goal of developing better intuition about how neural networks operate and use that to guide the simplest possible network design.
Our approach should be seen as complimentary to the one described in \cite{LearningToLearnScale} and related work.
In this vein our approach is similar to those taken by \cite{ShuffleNet2017, SpatialBottleneck2016} and allows to further improve the performance, while providing a glimpse on its internal operation.
Our network design is based on \mbox{MobileNetV1}~\cite{MobilenetV1}. It retains its simplicity and  does not require any special operators while significantly improves its accuracy, achieving state of the art on multiple image classification and detection tasks for mobile applications. 
\newcommand{\fig}[1]{Fig.~\ref{#1}}
\newcommand{\beveco}{Best viewed in color.}
\def\interior{\operatorname{interior}}

\section{Preliminaries, discussion and intuition}

\subsection{Depthwise Separable Convolutions}
Depthwise Separable Convolutions are a key building block for many efficient neural network architectures \cite{MobilenetV1,Chollet_2017_CVPR,ShuffleNet2017} and we use them in the present work as well.
The basic idea is to replace a full convolutional operator with a factorized version that splits convolution into two separate layers.
The first layer is called a depthwise convolution, it performs lightweight filtering by applying a single convolutional filter per input channel.
The second layer is a $1 \times 1$ convolution, called a pointwise convolution, which is responsible for building new features through computing linear combinations of the input channels. 

Standard convolution takes an $h_i\times w_i\times d_i$ input tensor $L_i$, and applies convolutional  kernel $K\in \real^{k\times k \times d_i \times d_j}$ to produce an $h_i\times w_i\times d_j$ output tensor $L_j$.
Standard convolutional layers have the computational cost of $h_i \cdot w_i \cdot d_i \cdot d_j \cdot k \cdot k$. %

Depthwise separable convolutions are a drop-in replacement for standard convolutional layers.
Empirically they work almost as well as regular convolutions but only cost:
\begin{equation}  
  h_i \cdot w_i \cdot d_i (k^2 + d_j)
  \label{eq:depthwise}
\end{equation}
which is the sum of the depthwise and $1 \times 1$ pointwise convolutions.
Effectively depthwise separable convolution reduces computation compared to traditional layers by almost a factor of $k^2$\footnote{more precisely, by a factor $k^2 d_j/(k^2 + d_j)$}.
\mbox{MobileNetV2} uses $k=3$ ($3 \times 3$ depthwise separable convolutions) so the computational cost is $8$ to $9$ times smaller than that of standard convolutions at only a small reduction in accuracy \cite{MobilenetV1}.

\subsection{Linear Bottlenecks}
\label{sec:bottlenecks}
Consider a deep neural network consisting of $n$ layers $L_i$ each of which has an activation tensor of dimensions $h_i \times w_i \times d_i$.
Throughout this section we will be discussing the basic properties of these
activation tensors, which we will treat as containers of $h_i \times
w_i$ ``pixels'' with $d_i$ dimensions.
Informally, for an input set of real images, we say that the set of layer activations (for any layer $L_i$) forms a ``manifold of interest''. It has been long assumed that manifolds of interest in neural networks could be embedded in low-dimensional subspaces.
In other words, when we look at all individual $d$-channel pixels of a deep 
convolutional layer, the information encoded in those values actually lie in some manifold, which in turn is embeddable into a low-dimensional subspace\footnote{Note that dimensionality of the manifold differs from the dimensionality of a subspace that could be embedded via a linear transformation.}.

At a first glance, such a fact could then be captured and exploited by simply reducing the dimensionality of a layer thus reducing the dimensionality of the operating space.
This has been successfully exploited by \mbox{MobileNetV1}~\cite{MobilenetV1} to effectively trade off between computation and accuracy via a width multiplier parameter, and has been incorporated into efficient model designs of other networks as well \cite{ShuffleNet2017}.
Following that intuition, the width multiplier approach allows one to reduce the dimensionality of the activation space until the manifold of interest spans
this entire space.
However, this intuition breaks down when we recall that deep convolutional neural networks actually have non-linear per coordinate transformations, such as $\relu$.
For example, $\relu$ applied to a line in 1D space produces a 'ray', where as in $\real^n$ space, it generally results in a piece-wise linear curve with $n$-joints.

It is easy to see that in general if a result of a layer transformation $\relu (B x)$ has a non-zero volume $S$, the points mapped to $\interior{S}$ are obtained via a linear transformation $B$ of the input, thus indicating that the part of the input space corresponding to the full dimensional output, is limited to a linear transformation.
In other words, deep networks only have the power of a linear classifier on the non-zero volume part of the output domain.
We refer to supplemental material for a more formal statement. 

On the other hand, when $\relu$ collapses the channel, it inevitably loses information in {\it that channel}. However if we have lots of channels, and there is a a structure in the activation manifold that information might still be preserved in the other channels. 
In supplemental materials, we show that if the input manifold can be embedded into a significantly lower-dimensional subspace of the activation space then the $\relu$ transformation preserves the information while introducing the needed complexity into the set of expressible functions.

\begin{figure}
    \includegraphics[width=.98\linewidth,keepaspectratio=true]{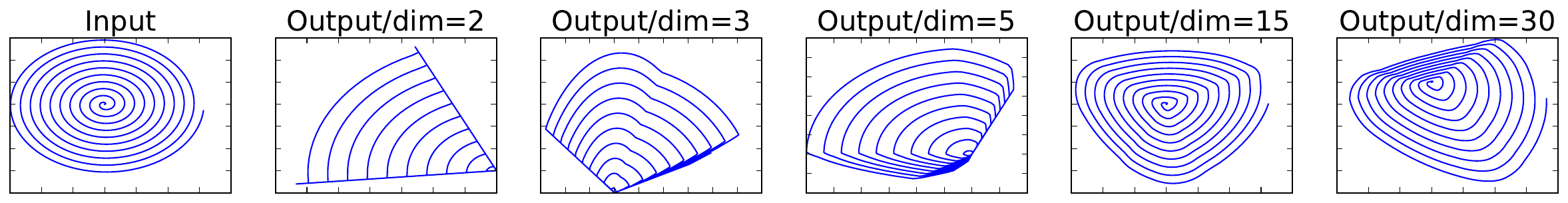}
    \caption{\small{
        Examples of $\relu$ transformations of low-dimensional manifolds embedded in
        higher-dimensional spaces.
        In these examples the initial spiral is embedded into an $n$-dimensional space using random matrix $T$ followed by $\relu$, and then projected back to the 2D space using $T^{-1}$. In examples above $n=2,3$ result in
        information loss where certain points of the manifold collapse into each other,
        while for $n=15$ to $30$ the  transformation is highly non-convex.}
    }
    \label{fig:a_relu_b_spiral}
\end{figure}

\begin{figure}[!t]
  \centering
  \begin{subfigure}{.2\textwidth}
    \centering
    \caption {Regular}
    \includegraphics[width=.6\linewidth,keepaspectratio=true]{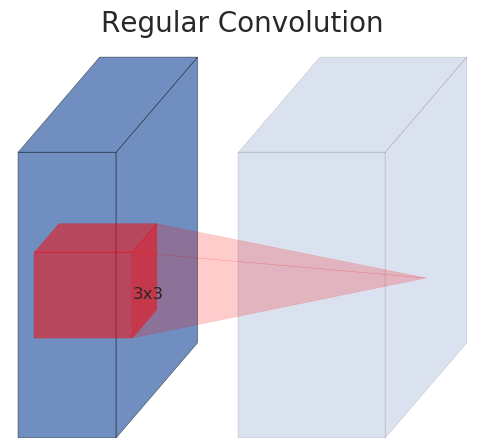}
  \end{subfigure}
  \begin{subfigure}{.2\textwidth}
    \caption {Separable}
    \includegraphics[width=\linewidth,keepaspectratio=true]{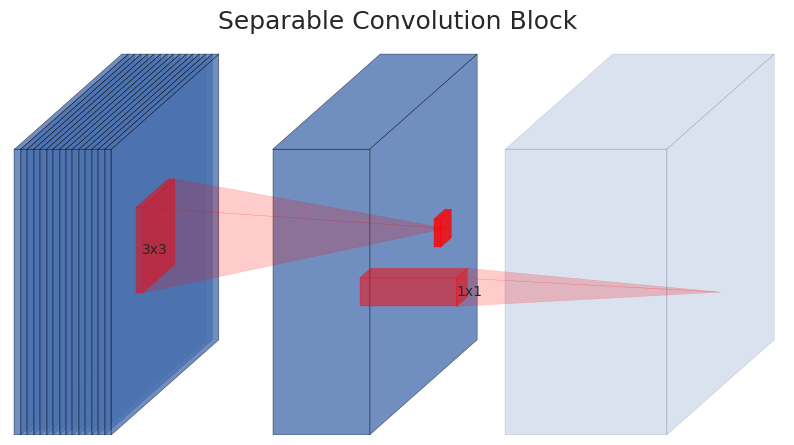}
  \end{subfigure}

\vspace{10pt}
  \begin{subfigure}{.2\textwidth}
    \caption{Separable with linear bottleneck}
    \label{fig:bneck}
    \includegraphics[width=\linewidth,keepaspectratio=true]{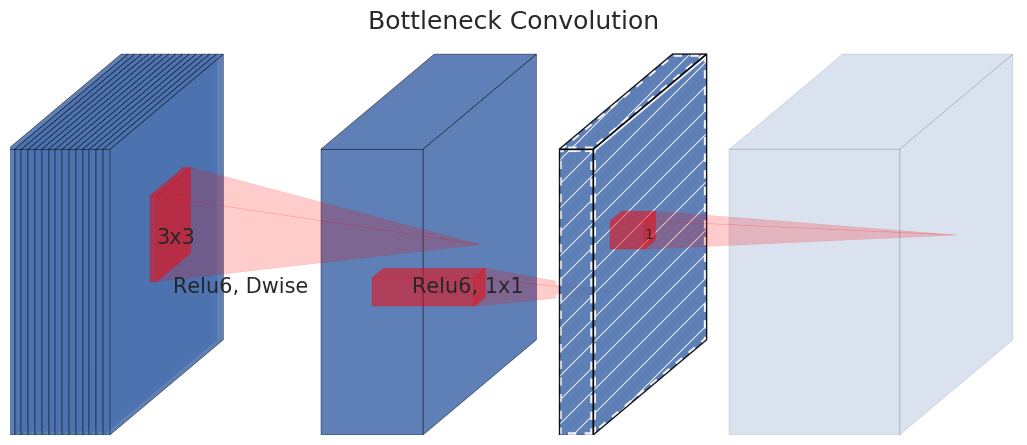}
  \end{subfigure} \hspace{15pt}
  \begin{subfigure}{.2\textwidth}
    \caption{Bottleneck with expansion layer}
    \label{fig:exp}
    \includegraphics[width=\linewidth,keepaspectratio=true]{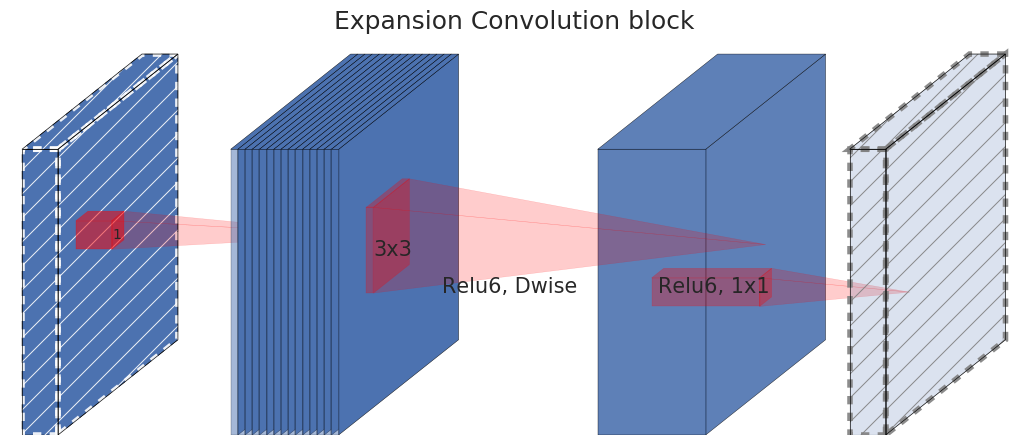}
  \end{subfigure}
  \caption{\small{
    Evolution of separable convolution blocks. The diagonally hatched texture indicates
    layers that do not contain non-linearities.
    The last (lightly colored) layer indicates the beginning of the next block. Note: ~\ref{fig:exp} and
    ~\ref{fig:bneck} are equivalent blocks when stacked. \beveco }
  }
\end{figure}

To summarize, we have highlighted two properties that are indicative of 
the requirement that the manifold of interest should lie in a low-dimensional 
subspace of the higher-dimensional activation space:

\begin{enumerate}
  \item If the manifold of interest remains non-zero volume after $\relu$ transformation, it corresponds to a linear transformation.
  \item $\relu$ is capable of preserving complete information about the input manifold, but only if the input manifold lies in a low-dimensional subspace of the input space.
\end{enumerate}

These two insights provide us with an empirical hint for optimizing existing neural architectures: assuming the manifold of interest is low-dimensional we can capture this by inserting {\em linear bottleneck} layers into the convolutional blocks.
Experimental evidence suggests that using linear layers is crucial as it prevents non-linearities from destroying too much information.
In Section~\ref{sec:experiments}, we show empirically that using non-linear layers in bottlenecks indeed hurts the performance by several percent, further validating our hypothesis\footnote{We note that in the presence of shortcuts the information loss is actually less strong.}. We note that similar reports where non-linearity was helped were reported in \cite{DeepPyramidalNetworks} where non-linearity was removed from the input of the traditional residual block and that lead to improved performance on CIFAR dataset.

For the remainder of this paper we will be utilizing bottleneck convolutions.
We will refer to the ratio between the size of the input bottleneck and the inner size as the \emph{expansion ratio}.

\subsection{Inverted residuals}
\begin{figure}
  \begin{subfigure}[t]{.22\textwidth}
    \caption {Residual block}
    \includegraphics[width=\linewidth,keepaspectratio=true]{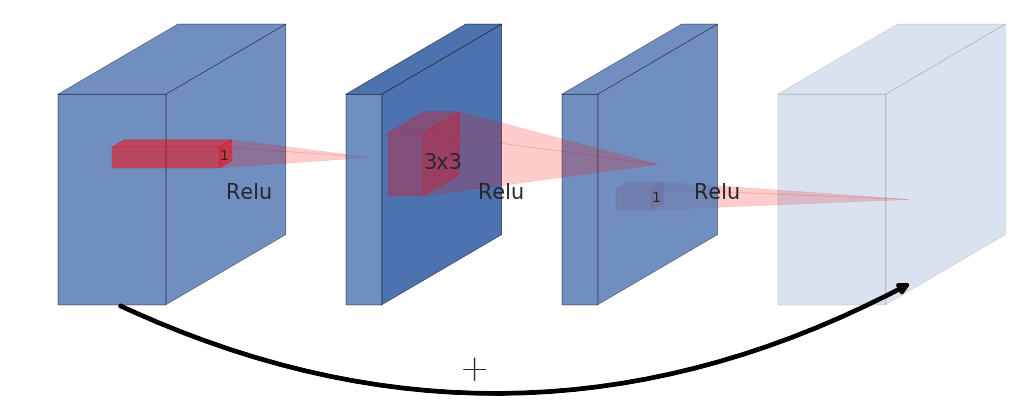}
  \end{subfigure}
  \begin{subfigure}[t]{.22\textwidth}
    \caption {Inverted residual block}
    \includegraphics[width=\linewidth,keepaspectratio=true]{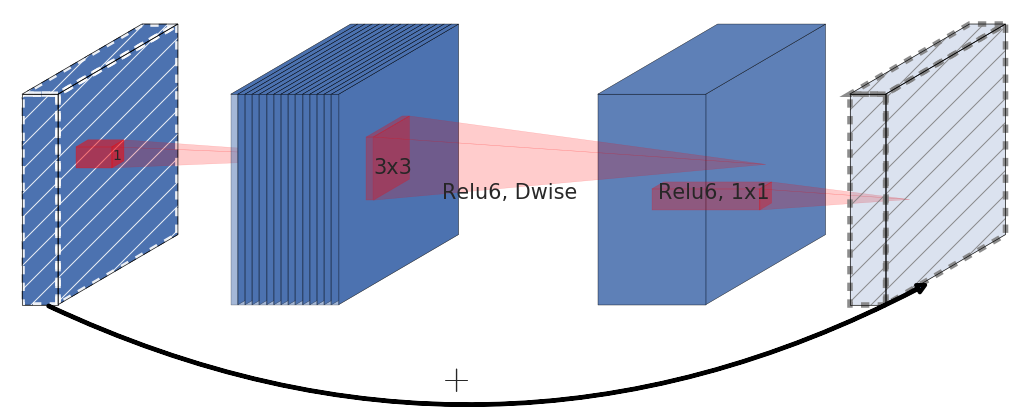}
    \label{fig:bottleneck-residual-block}
  \end{subfigure}
  \caption{
    The difference between residual block \cite{ResNet, ResNext2016} and inverted residual.
    Diagonally hatched layers do not use non-linearities.
    We use thickness of each block to indicate its relative number of channels.
    Note how classical residuals connects the layers with high number of channels, whereas the inverted residuals connect the bottlenecks. \beveco
  }
  \label{fig:bottleneck_with_residuals} 
\end{figure}

The bottleneck blocks appear similar to residual block where each block contains an input followed by several bottlenecks then followed by expansion \cite{ResNet}.
However, inspired by the intuition that the bottlenecks actually contain all the necessary information, while an expansion layer acts merely as an implementation detail that accompanies a non-linear transformation of the tensor, we use shortcuts directly between the bottlenecks.
Figure~\ref{fig:bottleneck_with_residuals} provides a schematic visualization of the difference in the designs.
The motivation for inserting shortcuts is similar to that of classical residual connections: we want to improve the ability of a gradient to propagate across  multiplier layers.
However, the inverted design is considerably more memory efficient (see Section~\ref{sec:implementation} for details), as well as works slightly better in our experiments. 

\paragraph{Running time and parameter count for bottleneck convolution}
The basic implementation structure is illustrated in Table~\ref{fig:bottlenec_block_table}.
For a block of size $h\times w$, expansion factor $t$ and kernel size $k$ with $d'$ input channels and $d''$ output channels, the total number of multiply add required is 
$h \cdot w \cdot d' \cdot t(d' + k^2 + d'')$.
Compared with \eqref{eq:depthwise} this expression has an extra term, as indeed we have an extra $1\times 1$ convolution, however the nature of our networks allows us to utilize much smaller input and output dimensions.
In Table~\ref{tab:my_label} we compare the needed sizes for each resolution between \mbox{MobileNetV1}, \mbox{MobileNetV2} and \mbox{ShuffleNet}.

\subsection{Information flow interpretation}
\label{sec:information-flow}
One interesting property of our architecture is that it provides a natural separation between the input/output {\em domains} of the building blocks (bottleneck layers), and the {\em layer transformation} -- that is a non-linear function that converts input to the output.
The former can be seen as the {\em capacity} of the network at each layer, whereas the latter as the {\em expressiveness}.
This is in contrast with traditional convolutional blocks, both regular and separable, where both expressiveness and capacity are tangled together and are functions of the output layer depth.

In particular, in our case, when inner layer depth is $0$ the underlying convolution is the identity function thanks to the shortcut connection.
When the expansion ratio is smaller than $1$, this is a classical residual convolutional block \cite{ResNet,ResNext2016}.
However, for our purposes we show that expansion ratio greater than $1$ is the most useful.

This interpretation allows us to study the expressiveness of the network separately from its capacity and we believe that further exploration of this separation is warranted to provide a better understanding of the network properties.

\section{Model Architecture}

Now we describe our architecture in detail. As discussed in the previous section the
basic building block is a bottleneck depth-separable convolution with residuals.
The detailed structure
of this block is shown in Table~\ref{fig:bottlenec_block_table}. The architecture of \mbox{MobileNetV2} contains the initial fully convolution layer with $32$ filters, followed by $19$ {\em residual bottleneck} layers described in the Table~\ref{mobilenet:arch}. We use $\relus$ as the non-linearity because of its robustness when used with low-precision computation \cite{MobilenetV1}. We always use kernel size $3\times 3$ as is standard for modern networks, and utilize dropout and batch normalization during training. 

With the exception of the first layer, we use constant expansion rate throughout the network. In our experiments we find that expansion rates between $5$ and $10$ result in nearly identical performance curves, with smaller networks being better off with slightly smaller expansion rates and larger networks having slightly better performance with larger expansion rates.

For all our main experiments we use expansion factor of $6$ applied to the size of the input tensor. For example, for a bottleneck layer that takes $64$-channel input tensor and produces a tensor with $128$ channels, the intermediate expansion layer is then $64 \cdot 6 = 384$ channels.

\begin{table}
\centering
    \begin{tabular}{c|c|c}
    Input & Operator & Output\\
    \toprule [0.2em]
    $h \times w \times k$ & 1x1 \conv, \relus & $h \times w \times (tk)$\\
    $h \times w \times tk$& 3x3 \depthwise s=$s$, \relus & $\frac{h}{s} \times \frac{w}{s} \times (tk)$\\ 
    $\frac{h}{s} \times \frac{w}{s} \times tk$ & linear 1x1 \conv & $\frac{h}{s} \times \frac{w}{s} \times k'$\\
    \toprule[0.2em]
    \end{tabular}

    \caption{{\em Bottleneck residual block} transforming from $k$ to $k'$ channels, with stride $s$, and expansion factor $t$.}

\label{fig:bottlenec_block_table}
\end{table}

\begin{table}[t]
\centering
\vspace{0pt}
\begin{tabular}{c|c|c|c|c|c}
\toprule[0.2em]
Input & Operator                           & $t$& $c$ & $n$ & $s$ \\
\toprule[0.2em]
$224^2\times3$ &    conv2d                  &  - &  32 & 1 & 2\\
$112^2\times32$ &    bottleneck    &  1 & 16   & 1 & 1 \\
$112^2\times16$ &   bottleneck     &  6 & 24   & 2 & 2 \\
$56^2\times24$ &   bottleneck      &  6 & 32   & 3 & 2 \\
$28^2\times32$ & bottleneck        &  6 & 64   & 4 & 2 \\
$14^2\times64$ &    bottleneck     &  6 & 96   & 3 & 1 \\
$14^2\times96$ &    bottleneck     &  6 & 160  & 3 & 2 \\
$7^2\times160$ &    bottleneck     &  6 & 320  & 1 & 1 \\
$7^2\times 320$ & conv2d 1x1       &  - & 1280 & 1 & 1 \\
$7^2\times1280$  & avgpool 7x7     &  - & -    & 1 & - \\   
$1\times1\times 1280$ & conv2d 1x1                      &  - & k    & -& \\
\toprule[0.2em]
\end{tabular}

\caption {
    \mbox{MobileNetV2} : Each line describes a sequence of 1 or more identical (modulo stride)  layers, repeated $n$ times.
    All layers in the same sequence have the same number $c$ of output channels.
    The first layer of each sequence has a stride $s$ and all others use stride $1$.
    All spatial convolutions use $3\times 3$ kernels. The expansion factor $t$ is always applied to the input size as described in Table~\ref{fig:bottlenec_block_table}.
}
\label{mobilenet:arch}
\end{table}

\begin{table}[t]
    \centering
\vspace{0pt}
\scalebox{0.75}{
\begin{tabular}{c|c|c|c}
\toprule[0.2em]
Size  &  MobileNetV1  &  MobileNetV2  &  ShuffleNet 
\\
& & & (2x,g=3) %
\\
\toprule[0.2em]
112x112 & 64/1600  &  16/400  & 32/800   %
\\
56x56  & 128/800  & 32/200   & 48/300    %
\\
28x28  & 256/400  & 64/100   &  400/600K %
\\
14x14  & 512/200  & 160/62   &  800/310  %
\\
7x7  &  1024/199   & 320/32  &  1600/156 %
\\
1x1  &  1024/2     & 1280/2  &  1600/3   %
\\

\toprule[0.1em]
{\bf max}  & 1600K & {\bf 400K}  & 600K %
\\
\toprule[0.2em]
\end{tabular}
}
\caption{
    The max number of channels/memory (in Kb) that needs to be materialized at each spatial resolution for different architectures.
    We assume 16-bit floats for activations.
    For \mbox{ShuffleNet}, we use $2x, g=3$ that matches the performance of \mbox{MobileNetV1} and \mbox{MobileNetV2}.
    For the first layer of  \mbox{MobileNetV2} and  \mbox{ShuffleNet} we can employ the trick described in Section~\ref{sec:implementation} to reduce memory requirement.
    Even though \mbox{ShuffleNet} employs bottlenecks elsewhere, the non-bottleneck tensors still need to be materialized due to the presence of shortcuts between the non-bottleneck tensors.
}
\label{tab:my_label}
\end{table}

\paragraph{Trade-off hyper parameters}
As in \cite{MobilenetV1} we tailor our architecture to different performance points, by using the input image resolution  and width multiplier as tunable hyper parameters, that can be adjusted depending on desired accuracy/performance trade-offs. Our primary network (width multiplier $1$, $224\times 224$), has a computational cost of 300 million multiply-adds and uses 3.4 million parameters. We explore the performance trade offs, for input resolutions from $96$ to $224$, and width multipliers of $0.35$ to $1.4$.
The network computational cost ranges from $7$ multiply adds to 585M MAdds, while the model size vary between 1.7M and 6.9M parameters.

One minor implementation difference, with \cite{MobilenetV1} is that for multipliers less than one, we apply width multiplier to all layers except the very last convolutional layer.
This improves performance for smaller models.

\section{Implementation Notes}

\begin{figure}
 \centering
 \begin{subfigure}[t]{0.31\textwidth}
  \centering
  \includegraphics[width=1.0\linewidth,keepaspectratio=true]{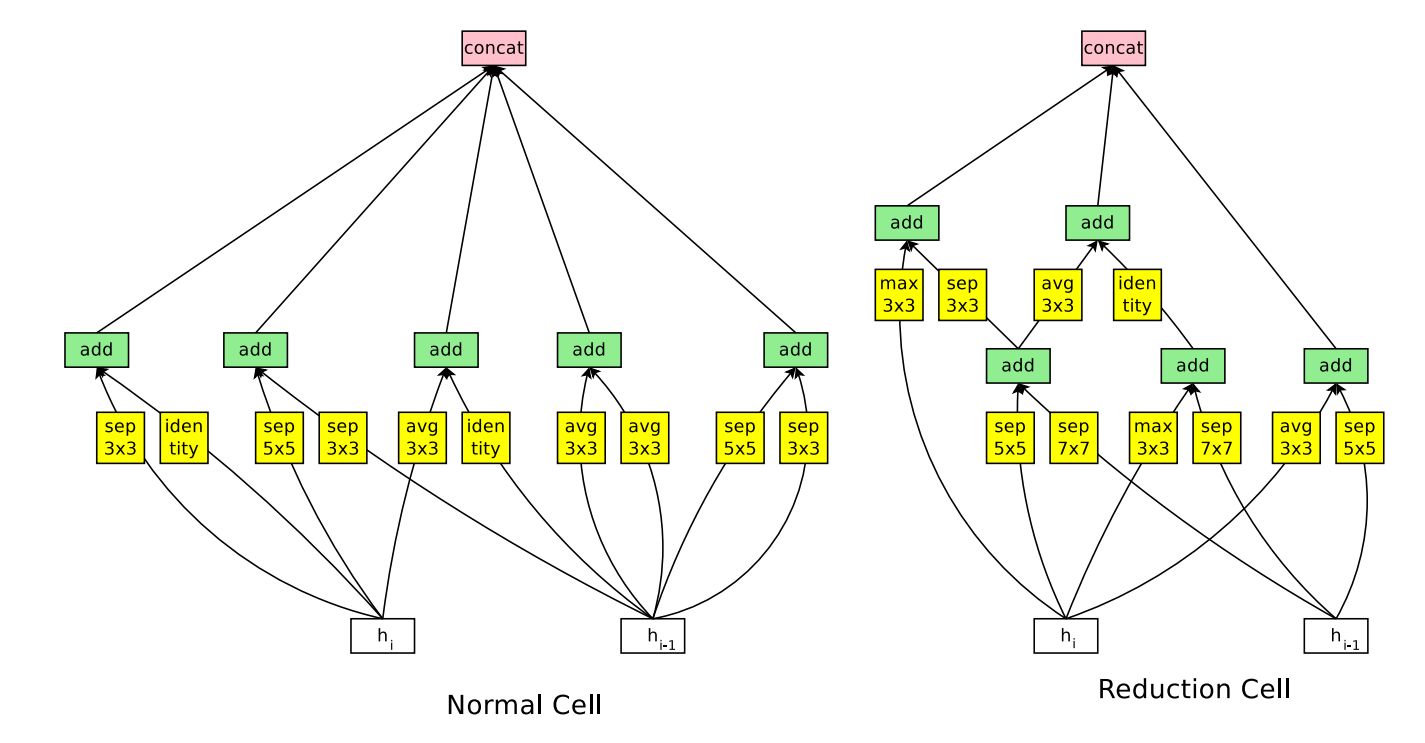}
  \caption{NasNet\cite{LearningToLearnScale}}
  \end{subfigure}
  \begin{subfigure}[t]{0.15\textwidth}
  \centering
  \includegraphics[width=1.0\linewidth,keepaspectratio=true]
  {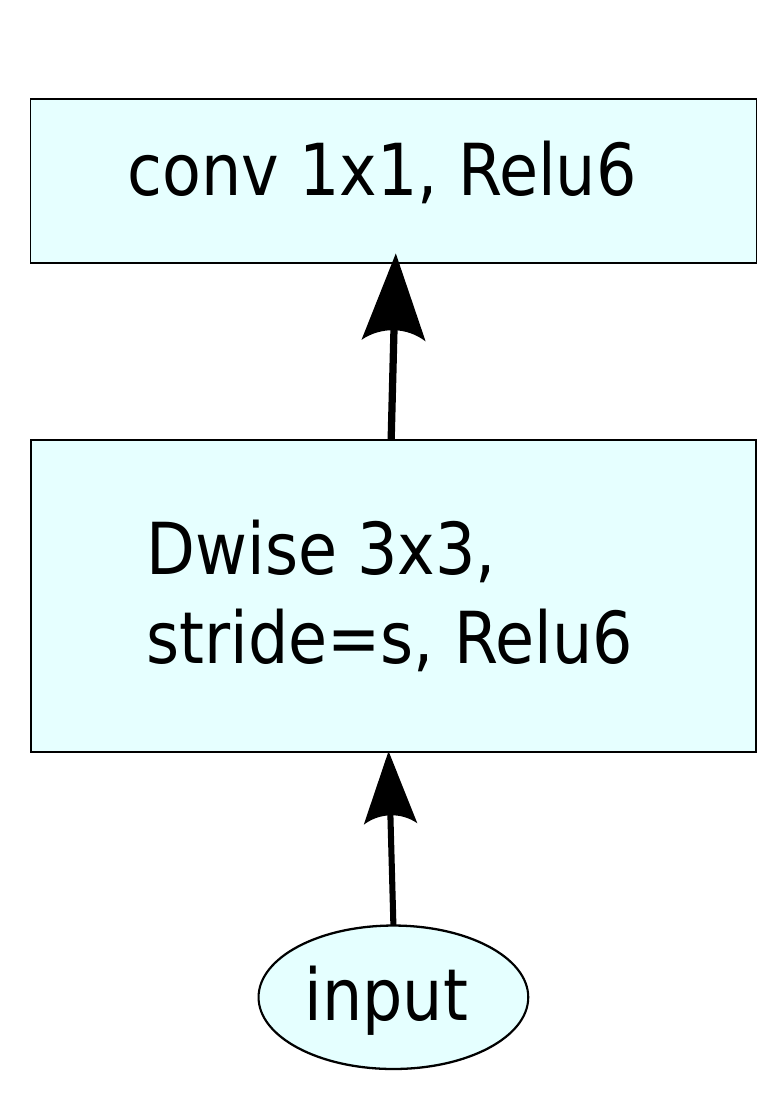}
  \caption{MobileNet\cite{MobilenetV1}}
  \end{subfigure} 
  \vspace{5pt}
  \begin{subfigure}[b]{0.23\textwidth}
  \centering
  \includegraphics[width=1.0\linewidth,keepaspectratio=true]{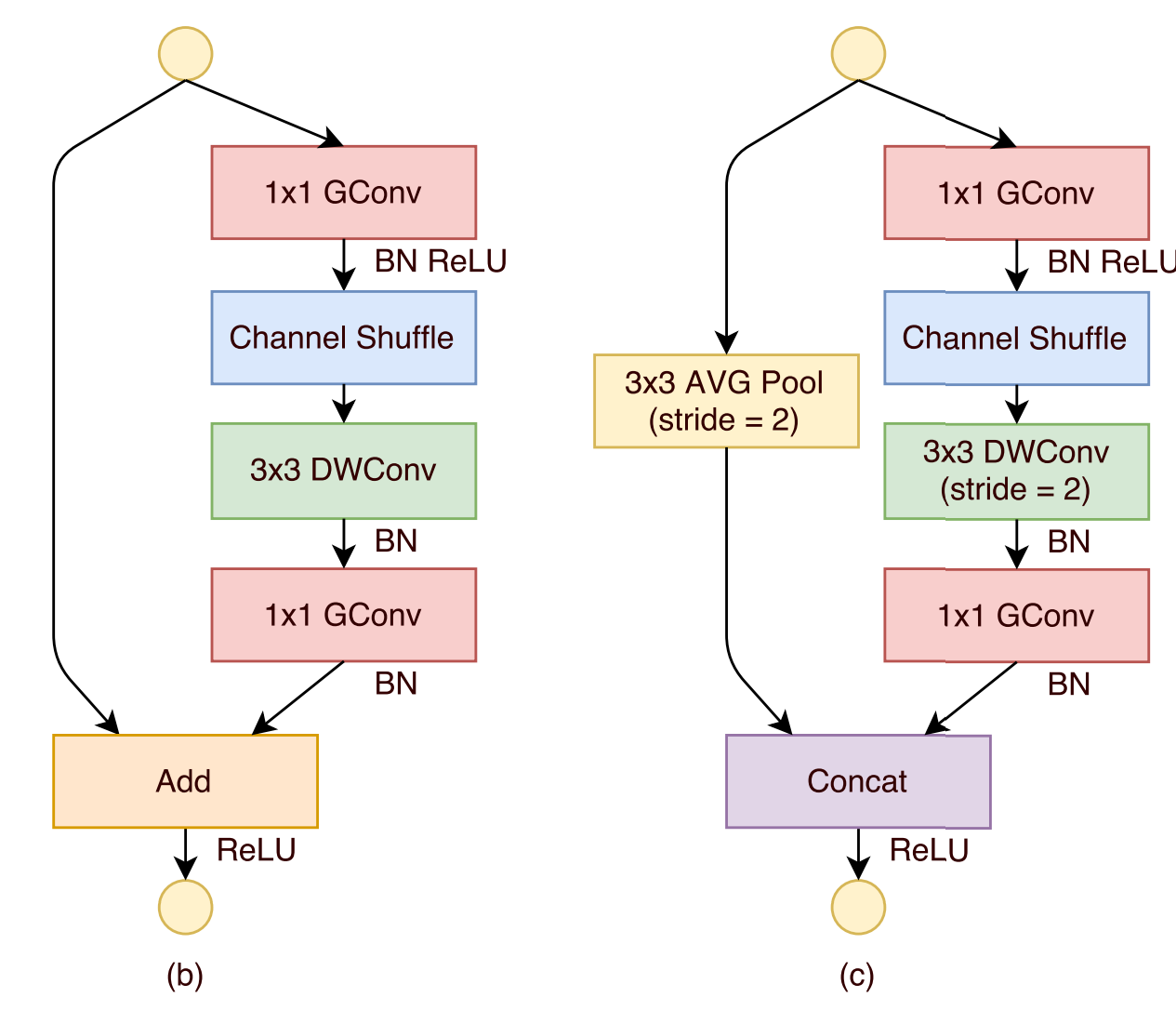}
  \caption{ShuffleNet~\cite{ShuffleNet2017}}
  \end{subfigure}
 \begin{subfigure}[b]{0.23\textwidth}
 \centering
  \includegraphics[width=1.0\linewidth,keepaspectratio=true]{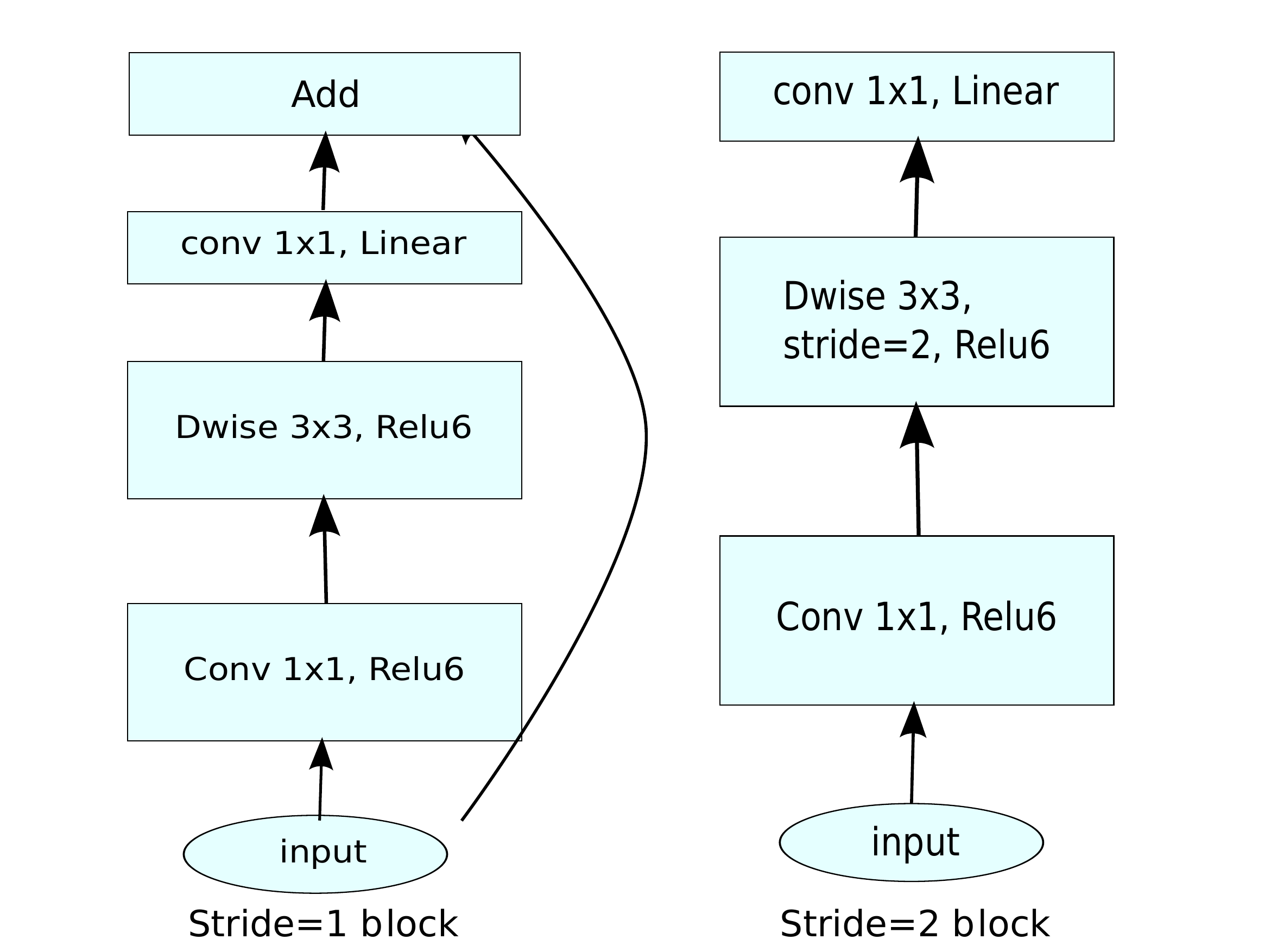}
 \caption{Mobilenet V2} 
  \end{subfigure}
  \caption{Comparison of convolutional blocks for different architectures. ShuffleNet uses Group Convolutions \cite{ShuffleNet2017} and shuffling, it  also  uses conventional residual approach   where inner blocks are narrower than output. ShuffleNet and NasNet illustrations are from respective papers.}
\end{figure}
\label{sec:implementation}

\subsection{Memory efficient inference}
The inverted residual bottleneck layers allow a particularly memory efficient implementation which is very important for mobile applications. A standard efficient implementation of inference that uses for instance
TensorFlow\cite{tensorflow2015-whitepaper} or Caffe~\cite{caffe},  builds a directed acyclic compute hypergraph $G$, consisting of edges representing the operations and nodes representing tensors of intermediate computation. The computation is scheduled in order to minimize the total number of tensors that needs to be stored in memory.
In the most general case, it searches over all plausible computation orders $\Sigma(G)$ and picks the one that minimizes
$$
 M(G) = \min_{\pi\in \Sigma(G)} \max_{i \in 1..n} \left[\sum_{A \in R(i, \pi, G)} |A|\right] + \text{size}(\pi_i).
$$
where $R(i, \pi, G)$ is the list of intermediate tensors that are connected to any of $\pi_{i}\dots \pi_{n}$ nodes, $|A|$ represents the size of the tensor $A$ and
$size(i)$ is the total amount of memory needed for internal storage during operation $i$.

For graphs that have only trivial parallel structure (such as residual connection), there is only one non-trivial feasible computation order, and thus the total amount and a  bound  on the memory needed for inference on compute graph $G$ can be simplified: 
\begin{equation}
\label{eq:naive-lower-bound-mu}
M(G) = \max_{op \in G} \left[\sum_{A \in \text{op}_{inp}} |A| + \sum_{B \in \text{op}_{out}} |B| + |op|\right]
\end{equation}
Or to restate, the amount of memory is simply the maximum total size of combined inputs and outputs across all operations. In what follows we show that if we treat a bottleneck residual block as a single operation (and treat inner convolution
as a disposable tensor), the total amount of memory would be dominated by the size of bottleneck tensors, rather than the size of tensors that are internal to bottleneck (and much larger). 

\paragraph{Bottleneck Residual Block}
\def\F{{\cal F}}
\def\N{{\cal N}}
A bottleneck block operator $\F(x)$ shown in Figure~\ref{fig:bottleneck-residual-block} can be expressed as a composition of three
operators $\F(x) = [A \circ \N \circ  B] x$, where $A$ is a linear
transformation $A:\real^{s \times s \times k} \rightarrow \real^{s \times s \times n}$,  $\N$ is a non-linear per-channel transformation: $\N: \real^{s \times s \times n} \rightarrow \real^{s' \times s' \times n}$, and $B$ is again a linear transformation to the output domain: $B: \real^{s' \times s' \times n} \rightarrow \real^{s' \times s' \times k'}$.

For our networks $\N =
\relus \circ \depthwise \circ \relus$, but the results apply to any per-channel transformation. Suppose the size of the input domain is $|x|$ and the size of the output domain is $|y|$, then the memory required to compute $F(X)$ can be as low as $|s^2 k| + |s'^2 k'| + O(\max(s^2, s'^2))$.

The algorithm is based on the fact that the inner tensor $\cal I$ can be represented as concatenation of $t$ tensors, of size $n/t$ each 
and our function can then be represented as 
$$\F(x) = \sum_{i=1}^t (A_i \circ N \circ B_i)(x)$$ 
by  accumulating the sum, we only require one intermediate block of size $n/t$ to be kept in memory at all times. Using $n=t$ we end up having to keep only a single channel of the intermediate representation at all times. The two constraints that enabled us to use this trick is (a) the fact that the inner transformation (which includes non-linearity and depthwise) is per-channel, and (b) the consecutive non-per-channel operators have significant ratio of the input size to the output. For most of the traditional neural networks, such trick would not produce a significant improvement. 

We note that, the  number of multiply-adds operators needed to compute $F(X)$ using $t$-way split is independent of $t$, however in existing implementations we find that replacing one matrix multiplication with several smaller ones hurts runtime performance due to increased cache misses. We find that this approach is the most helpful to be used with $t$ being a small constant between $2$ and $5$. It significantly reduces the memory requirement, but still allows one to utilize most of the efficiencies gained by using highly optimized matrix multiplication and convolution operators provided by deep learning frameworks. It remains to be seen if special framework level optimization may lead to further runtime improvements. 

\newcommand\omt[1]{}

\section{Experiments}
\label{sec:experiments}

\begin{figure}
    \centering
    \includegraphics[width=.99\linewidth,keepaspectratio=true]{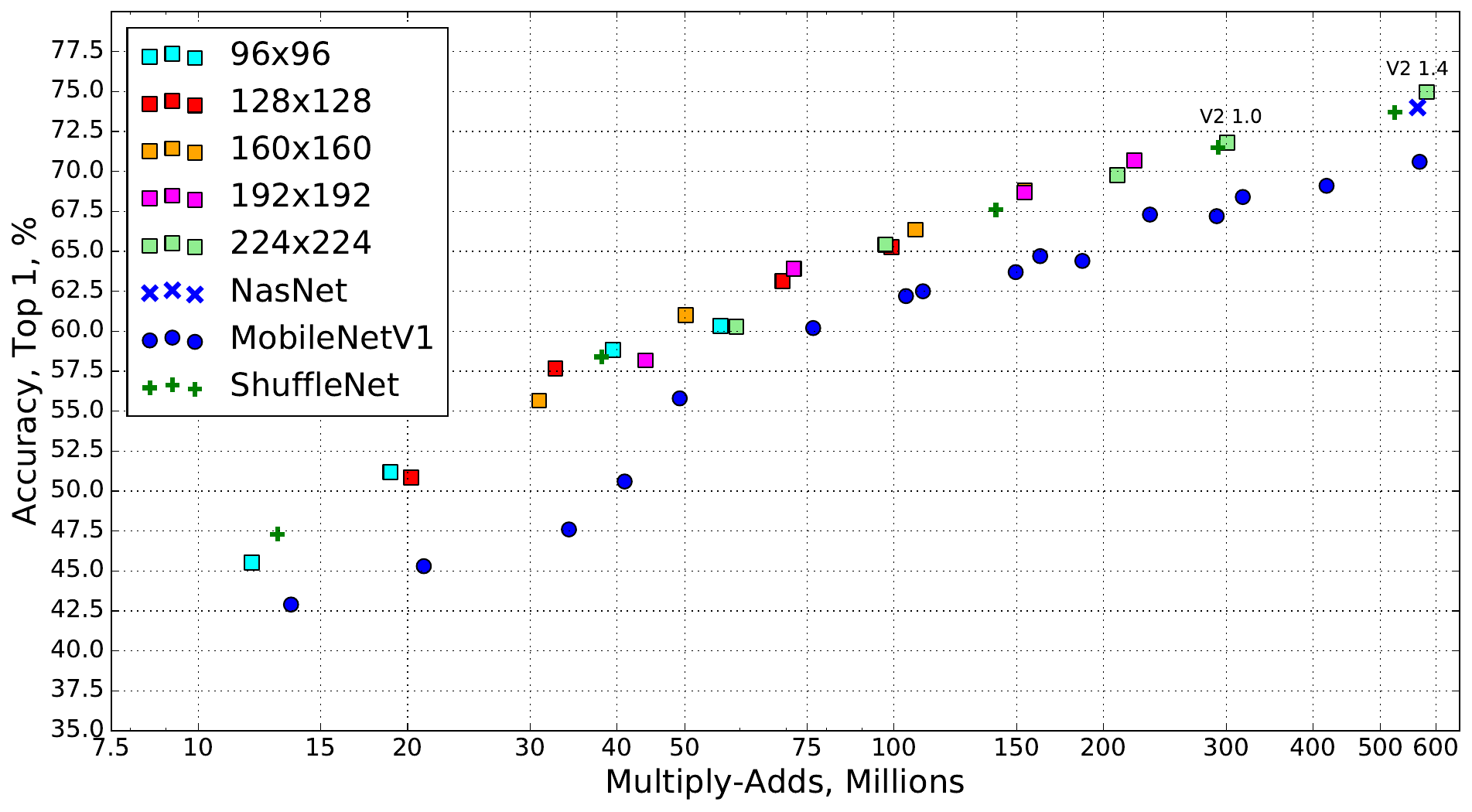}
    \caption{
        Performance curve of \mbox{MobileNetV2} vs \mbox{MobileNetV1}, \mbox{ShuffleNet}, \mbox{NAS}.
        For our networks we use multipliers $0.35$, $0.5$, $0.75$, $1.0$ for all resolutions, and additional $1.4$ for for $224$. \beveco
    }
    \label{fig:performance_curve}
\end{figure}

\begin{figure}
    \centering
    \begin{subfigure}{0.23\textwidth}
        \centering
        \includegraphics[width=.98\linewidth,keepaspectratio=true]{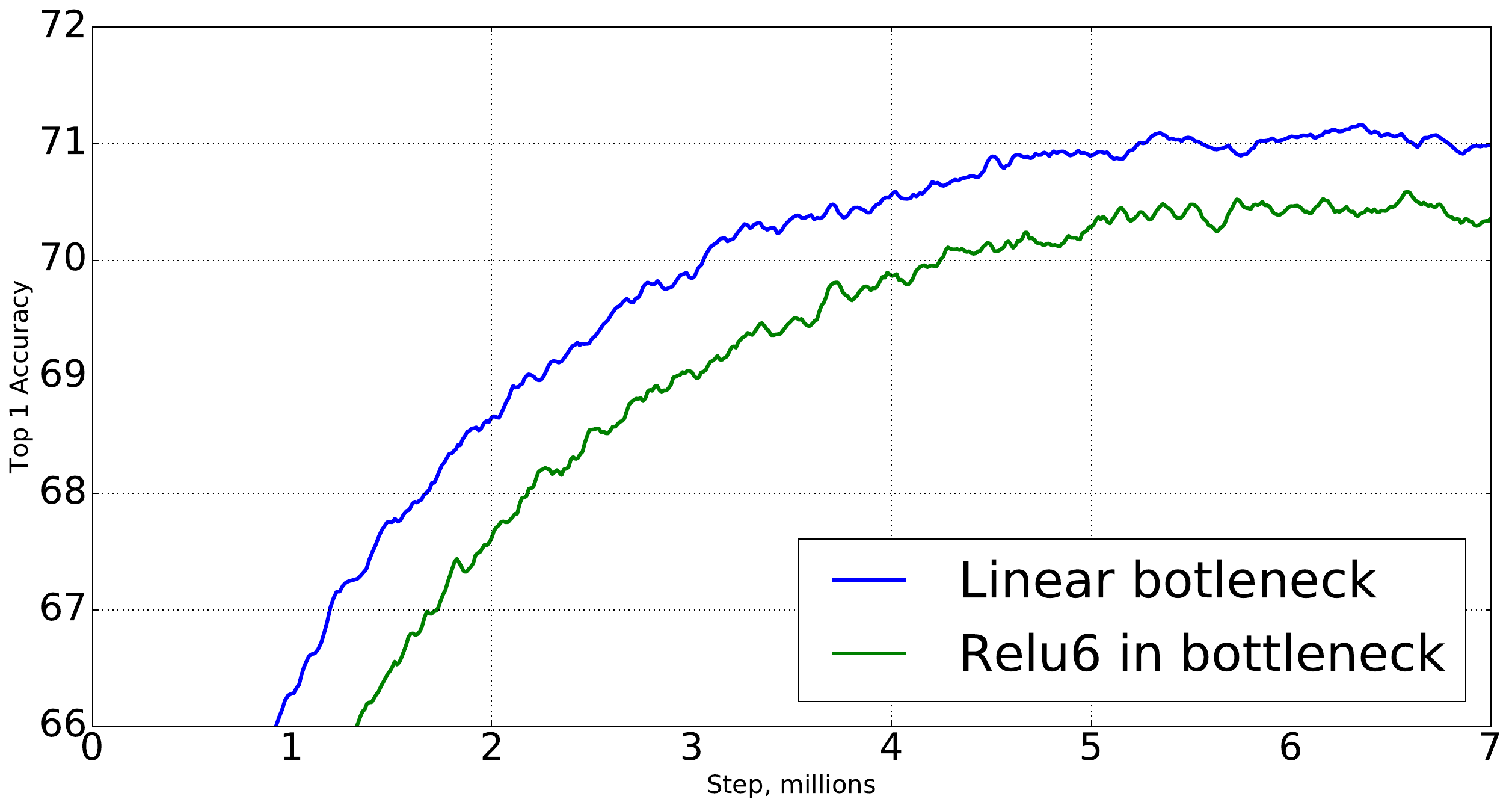}
        \caption{Impact of non-linearity in the bottleneck layer.}
        \label{fig:linearity-impact}
    \end{subfigure}
    \begin{subfigure}{0.23\textwidth}
        \includegraphics[width=.98\linewidth,keepaspectratio=true]{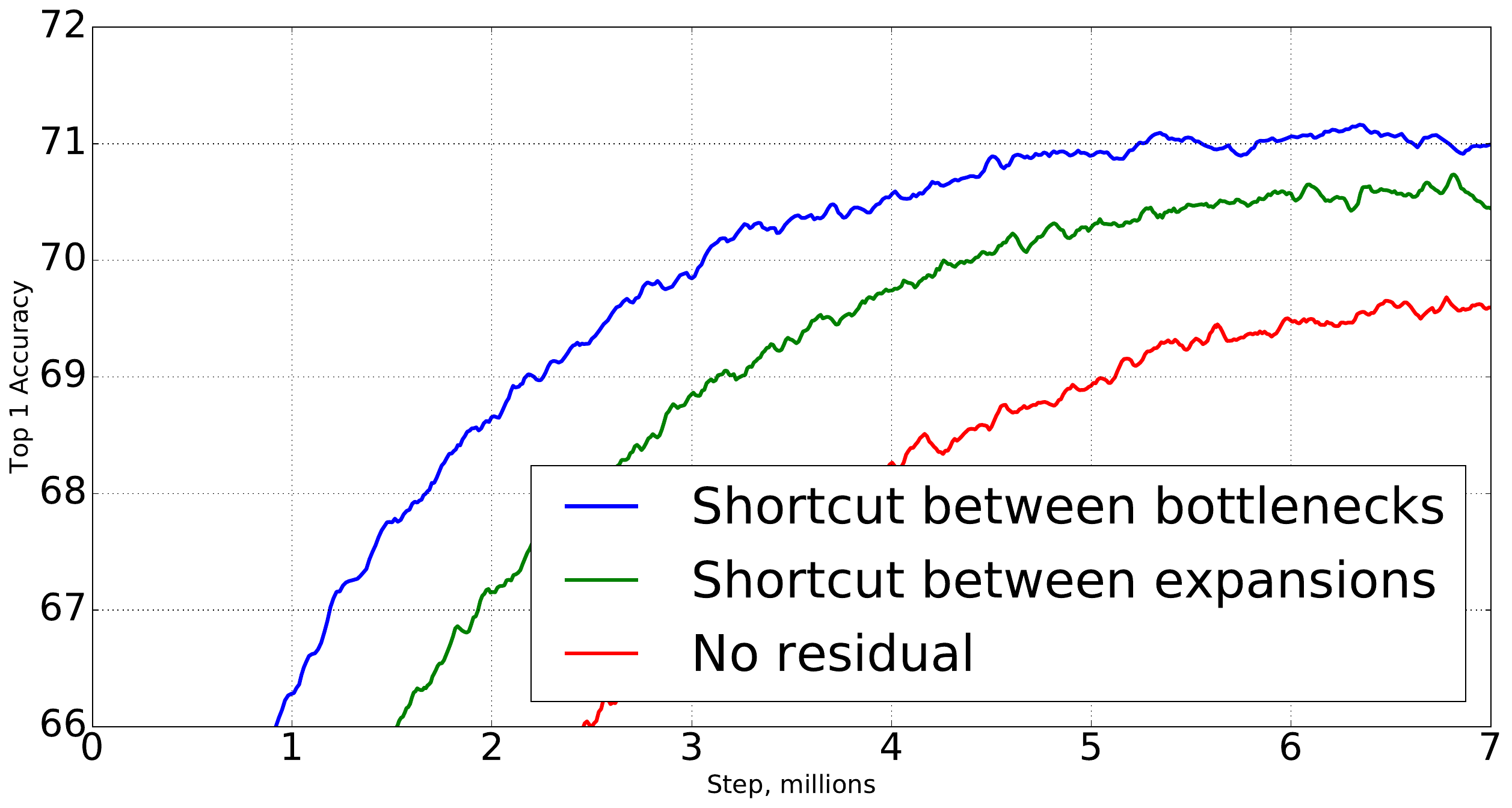}
        \caption{Impact of variations in residual blocks.}
        \label{fig:residual-impact}
    \end{subfigure}
    \caption{
        The impact of non-linearities and various types of shortcut (residual) connections.
    }
\end{figure}

\omt{
\begin{figure}
    \centering
    \begin{subfigure}[t]{0.23\textwidth}
        \includegraphics[width=.98\linewidth,keepaspectratio=true]{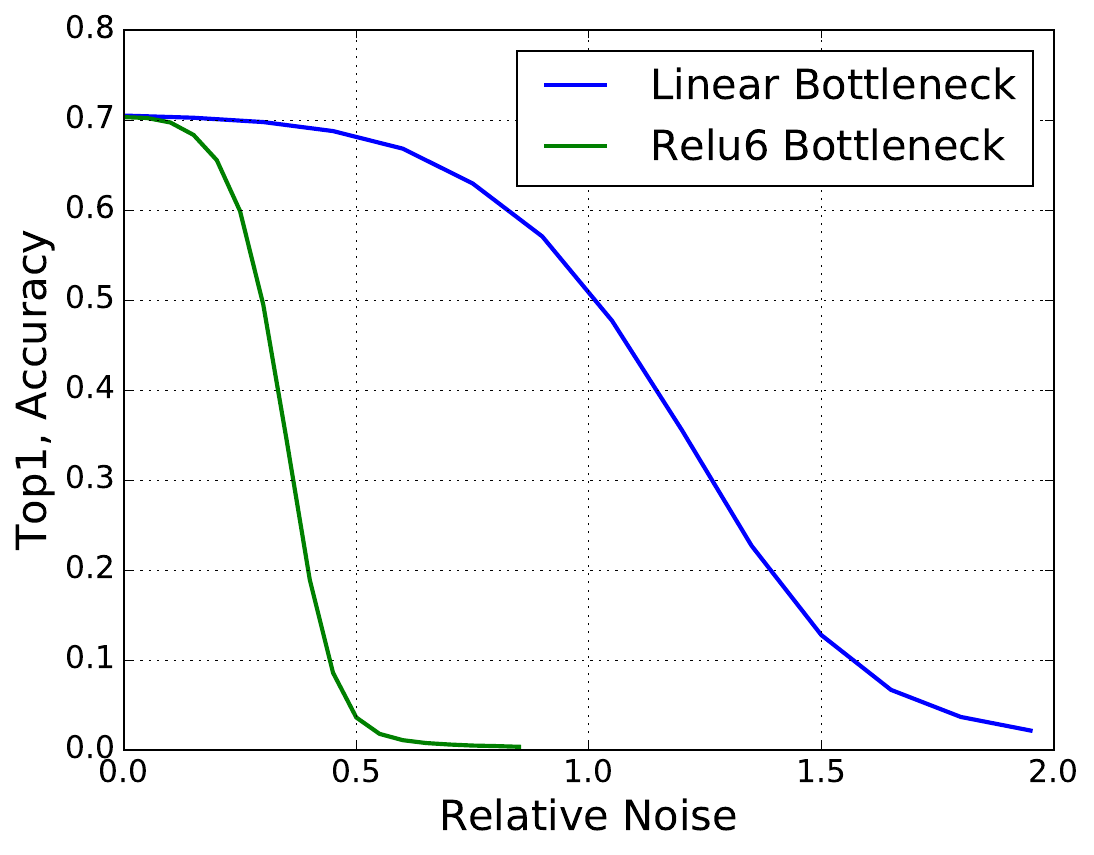}
        \caption{Accuracy vs. noise.}
        \label{fig:robustness_to_noise}
    \end{subfigure}
    \begin{subfigure}[t]{0.23\textwidth}
        \includegraphics[width=.98\linewidth,keepaspectratio=true]{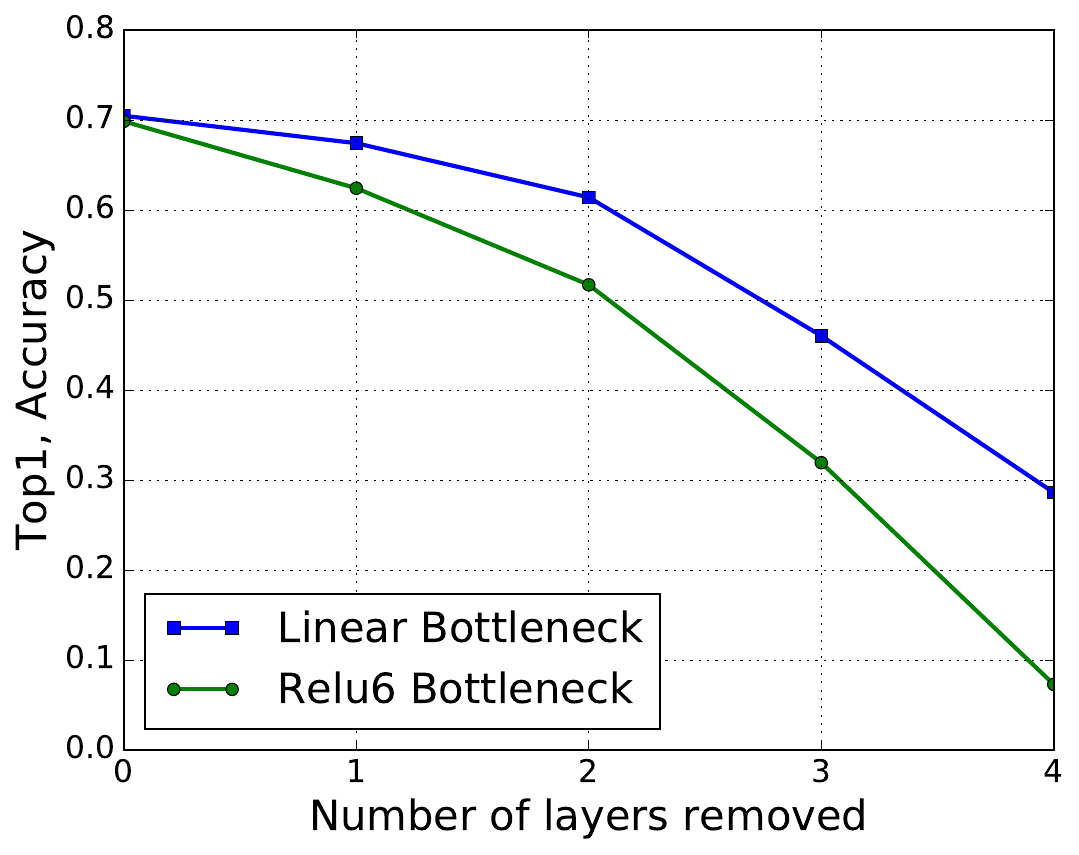}
        \caption{Accuracy vs.\# layers removed.}
        \label{fig:robustness_to_layer_removal}
    \end{subfigure}
    \caption{
        Robustness to noise.
        In Figure~\ref{fig:robustness_to_noise} we show how accuracy deteriorates as we add Gaussian noise to one layer.
        On $x$ axis we plot the relative amplitude of the noise vector vs. the signal vector.
        Note how linear bottleneck layers have much higher robustness.
        Figure~\ref{fig:robustness_to_layer_removal} shows that linear bottleneck are more robust to layer removal.
    }
\end{figure}
}

\subsection{ImageNet Classification}
\paragraph{Training setup}
We train our models using TensorFlow\cite{tensorflow2015-whitepaper}.
We use the standard RMSPropOptimizer with both decay and momentum set to $0.9$.
We use batch normalization after every layer, and the standard weight decay is set to $0.00004$.
Following \mbox{MobileNetV1}\cite{MobilenetV1} setup we use initial learning rate of $0.045$, and learning rate decay rate of $0.98$ per epoch.
We use 16 GPU asynchronous workers, and a batch size of $96$.

\paragraph{Results}
We compare our networks against \mbox{MobileNetV1}, \mbox{ShuffleNet} and \mbox{NASNet-A} models.
The statistics of a few selected models is shown in Table~\ref{table:performance-comparison} with the full performance graph shown in  Figure~\ref{fig:performance_curve}.
 
\begin{table}[b]
    \centering 
    \resizebox{\columnwidth}{!}{
        \begin{tabular}{l|c|cc|c}
        \toprule [0.2em]
        Network & Top 1 & Params & MAdds & CPU  \\
        \toprule [0.2em]
        MobileNetV1       &70.6 &  4.2M& 575M & 113ms  \\
        ShuffleNet (1.5)   & 71.5& {\bf 3.4M} & 292M & - \\
        ShuffleNet (x2)    &73.7 & 5.4M & 524M & - \\
        NasNet-A           & {74.0} & 5.3M& 564M & 183ms\\
        \toprule[0.1em]
        MobileNetV2       &\bf 72.0    & \bf 3.4M   & \bf 300M & {\bf 75ms}\\
        MobileNetV2 (1.4) & {\bf 74.7} & 6.9M& 585M & {\bf 143ms} \\
        \toprule [0.2em]
        \end{tabular}
    }
    \caption{
        Performance on \mbox{ImageNet}, comparison for different networks.
        As is common practice for ops, we count the total number of Multiply-Adds.
        In the last column we report running time in milliseconds (ms) for a single large core of the Google \mbox{Pixel 1} phone (using \mbox{TF-Lite}).
        We do not report \mbox{ShuffleNet} numbers as  efficient group convolutions and shuffling are not yet supported.
    }
    \label{table:performance-comparison}
\end{table}

\subsection{Object Detection}
We evaluate and compare the performance of \mbox{MobileNetV2} and \mbox{MobileNetV1} as feature extractors \cite{huang2016speed} for object detection with a modified version of the Single Shot Detector (SSD)~\cite{liu2016ssd} on \mbox{COCO} dataset \cite{COCO}.
We also compare to \mbox{YOLOv2}~\cite{redmon2016yolo9000} and original SSD (with VGG-16~\cite{VGGNet} as base network) as baselines.
We do not compare performance with other architectures such as \mbox{Faster-RCNN} \cite{ren2015faster} and RFCN \cite{dai2016rfcn} since our focus is on mobile/real-time models.

\textbf{SSDLite}: In this paper, we introduce a mobile friendly variant of regular SSD. We replace all the regular convolutions with separable convolutions (depthwise followed by $1\times1$ projection) in SSD prediction layers.
This design is in line with the overall design of \mbox{MobileNets} and is seen to be much more computationally efficient.
We call this modified version \mbox{SSDLite}.
Compared to regular SSD, \mbox{SSDLite} dramatically reduces both parameter count and computational cost as shown in Table~\ref{tab:ssd_lite}.

\begin{table}[!t]
  \centering
  \begin{tabular}{c | c c c}
  \toprule[0.2em]
   & Params & MAdds \\
  \toprule[0.2em]
  SSD\cite{liu2016ssd} & 14.8M &  1.25B  \\
  SSDLite & \textbf{2.1M} & \textbf{0.35B}  \\
  \toprule[0.2em]
  \end{tabular}
  \caption{
    Comparison of the size and the computational cost between SSD and SSDLite configured with \mbox{MobileNetV2} and making predictions for $80$ classes.}
  \label{tab:ssd_lite}
\end{table}

For \mbox{MobileNetV1}, we follow the setup in \cite{huang2016speed}.
For \mbox{MobileNetV2}, the first layer of SSDLite is attached to the  expansion of layer 15 (with output stride of 16). The second and the rest of SSDLite layers are attached on top of the last layer (with output stride of $32$). This setup is consistent with \mbox{MobileNetV1} as all  layers are attached to the  feature map of the same output strides.

Both \mbox{MobileNet} models are trained and evaluated with Open Source TensorFlow Object Detection API~\cite{detection-api}.
The input resolution of both models is $320\times 320$.
We benchmark and compare both mAP (COCO challenge metrics), number of parameters and number of Multiply-Adds.
The results are shown in Table~\ref{tab:mobilenet_ssd}.
\mbox{MobileNetV2} SSDLite is not only the most efficient model, but also the most accurate of the three.
Notably, \mbox{MobileNetV2} SSDLite is \textbf{$20\times$} more efficient and \textbf{$10\times$} smaller while still outperforms YOLOv2 on COCO dataset.

\begin{table}[!t]
  \centering
  \scalebox{0.85}{
      \begin{tabular}{c|c|cc|c}
          \toprule[0.2em]
          Network                            & mAP  &  Params   & MAdd & CPU\\
          \toprule[0.2em]
          SSD300\cite{liu2016ssd}          & 23.2 & 36.1M & 35.2B  & -  \\
          SSD512\cite{liu2016ssd}          & 26.8 & 36.1M & 99.5B  & -  \\
          YOLOv2\cite{redmon2016yolo9000}  & 21.6 & 50.7M & 17.5B  & -  \\
          MNet V1 + SSDLite           & 22.2 & 5.1M  & 1.3B    & 270ms\\
          \toprule[0.1em]
          MNet V2 + SSDLite           & 22.1 & \textbf{4.3M} & \textbf{0.8B} & 200ms  \\
          \toprule[0.2em]
      \end{tabular}
  }
  \caption{
      Performance comparison of \mbox{MobileNetV2} + SSDLite and other realtime detectors on the COCO dataset object detection task.
      \mbox{MobileNetV2} + SSDLite achieves competitive accuracy with significantly fewer parameters and smaller computational complexity.
      All models are trained on \mbox{\tt trainval35k} and evaluated on \mbox{\tt test-dev}.
      \mbox{SSD/YOLOv2} numbers are from \cite{redmon2016yolo9000}.
      The running time is reported for the large core of the Google \mbox{Pixel 1} phone, using an internal version of the \mbox{TF-Lite} engine.
  }
  \label{tab:mobilenet_ssd}
\end{table}

\subsection{Semantic Segmentation}
In this section, we compare \mbox{MobileNetV1} and \mbox{MobileNetV2} models used as feature extractors with \mbox{DeepLabv3}~\cite{DeepLabV3} for the task of mobile semantic segmentation.
\mbox{DeepLabv3} adopts atrous convolution~\cite{Holschneider1989real, Sermanet2013Overfeat, Papandreou2014untangling}, a powerful tool to explicitly control the resolution of computed feature maps, and builds five parallel heads including (a) Atrous Spatial Pyramid Pooling module (ASPP) \cite{DeepLabV2} containing three $3\times 3$ convolutions with different atrous rates, (b) $1\times 1$ convolution head, and (c) Image-level features \cite{ParseNet}.
We denote by $\emph{output\_stride}$ the ratio of input image spatial resolution to final output resolution, which is controlled by applying the atrous convolution properly.
For semantic segmentation, we usually employ $\emph{output\_stride}=16$ or $8$ for denser feature maps.
We conduct the experiments on the \mbox{PASCAL} VOC 2012 dataset \cite{PASCAL}, with extra annotated images from \cite{Hariharan2011semantic} and evaluation metric mIOU.

To build a mobile model, we experimented with three design variations: (1) different feature extractors, (2) simplifying the \mbox{DeepLabv3} heads for faster computation, and (3) different inference strategies for boosting the performance.
Our results are summarized in Table~\ref{tab:mobilenet_val}.
We have observed that: (a) the inference strategies, including multi-scale inputs and adding left-right flipped images, significantly increase the MAdds and thus are not suitable for on-device applications, (b) using $\emph{output\_stride}=16$ is more efficient than $\emph{output\_stride}=8$, (c) \mbox{MobileNetV1} is already a powerful feature extractor and only requires about $4.9 - 5.7$ times fewer MAdds than ResNet-101 \cite{ResNet} (\eg, mIOU: 78.56 $\emph{vs}$ 82.70, and MAdds: 941.9B $\emph{vs}$ 4870.6B), (d) it is more efficient to build \mbox{DeepLabv3} heads on top of the second last feature map of \mbox{MobileNetV2} than on the original last-layer feature map, since the second to last feature map contains $320$ channels instead of $1280$, and by doing so, we attain similar performance, but require about $2.5$ times fewer operations than the \mbox{MobileNetV1} counterparts, and (e) \mbox{DeepLabv3} heads are computationally expensive and removing the ASPP module significantly reduces the MAdds with only a slight performance degradation.
In the end of the Table~\ref{tab:mobilenet_val}, we identify a potential candidate for on-device applications (in bold face), which attains $75.32\%$ mIOU and only requires $2.75$B MAdds.

\begin{table}[!t]
  \centering
  \scalebox{0.81}{
  \begin{tabular}{c |c c c | c c c}
    \toprule[0.2em]
    Network &  OS  & ASPP       & MF        & mIOU & Params & MAdds\\
    \toprule[0.2em]
    MNet V1 &  16  & \checkmark &            & 75.29 & 11.15M &  14.25B\\  
            &   8  & \checkmark & \checkmark & 78.56 & 11.15M & 941.9B \\
    \toprule[0.1em]
    MNet V2* & 16  & \checkmark &            & 75.70 & 4.52M & 5.8B \\
             & 8   & \checkmark & \checkmark & 78.42 & 4.52M & 387B \\
    \toprule[0.1em]
    MNet V2* & 16  &            &            & {\bf 75.32} & {\bf 2.11M} & {\bf 2.75B} \\
             &  8  &            & \checkmark & 77.33 & 2.11M & 152.6B\\
    \toprule[0.1em]
    \toprule[0.1em]
    ResNet-101 & 16 & \checkmark &            & 80.49 & 58.16M & 81.0B \\
               &  8 & \checkmark & \checkmark & 82.70 & 58.16M & 4870.6B \\
    \bottomrule[0.1em]
  \end{tabular}
  }
  \caption{
    \mbox{MobileNet} + \mbox{DeepLabv3} inference strategy on the \mbox{PASCAL} VOC 2012 \textit{validation} set.
    {\bf MNet V2*}: Second last feature map is used for \mbox{DeepLabv3} heads, which includes (1) Atrous Spatial Pyramid Pooling ({\bf ASPP}) module, and (2) $1\times1$ convolution as well as image-pooling feature.
    {\bf OS}: \emph{output\_stride} that controls the output resolution of the segmentation map.
    {\bf MF}: Multi-scale and left-right flipped inputs during test. All of the models have been pretrained on COCO. The potential candidate for on-device applications is shown in bold face. \mbox{PASCAL} images have dimension $512\times 512$ and atrous convolution allows us to control output feature resolution without increasing the number of parameters.}
  \label{tab:mobilenet_val}
\end{table}

\omt{
\begin{table*}[!t]
  \centering
  \scalebox{1}{
  \begin{tabular}{c c |c c c c | c c}
    \toprule[0.2em]
    Feature Extractor & DeepLabv3 Heads &  OS    & MS         & Flip        & COCO & mIOU & FLOPS (B)\\
    \toprule[0.2em]
    MobileNetV1 & ASPP(6, 12, 18)         & 16            &            &             &      & 70.10 & 28.57\\
                & $1\times1$ conv         & 16            &            &             &  \checkmark    & 75.29 & 28.57\\
                & Image Pooling           & 8             &            &             &      & 70.47 & 106.57\\
                &                         & 8 & \checkmark &             &      & 72.79 & 991.93\\
                &                         & 8 & \checkmark & \checkmark  &      & 72.97 & 1983.80\\
                &                         & 8 &  \checkmark & \checkmark  & \checkmark & 78.56 & 1983.80\\
    \toprule[0.1em]
    MobileNetV2 & ASPP(6, 12, 18) &  16       &            &             &              & 70.26 & 30.39\\
                & $1\times1$ conv &  16       &            &             & \checkmark   & 76.26 & 30.39\\
                & Image Pooling   &   8 &            &             &      & 70.56 & 117.01\\
                &                 &   8 & \checkmark &             &      & 72.55 & 1099.47\\
                &                 &   8 & \checkmark & \checkmark  &      & 72.66 & 2198.87\\
                &                 &   8 & \checkmark & \checkmark  & \checkmark & 79.10 & 2198.87\\
    \toprule[0.1em]
    MobileNetV2* & ASPP(6, 12, 18) & 16    &            &             &                & 70.32 & 11.67 \\
                 & $1\times1$ conv & 16    &            &             & \checkmark     & 75.70 & 11.67 \\ 
                 & Image Pooling   &       8 &            &             &      & 70.51 & 42.30 \\
                 &                 &       8 & \checkmark &             &      & 72.10 & 387.56 \\
                 &                 &       8 & \checkmark & \checkmark  &      & 72.15 & 775.09 \\
                 &                 &       8 & \checkmark & \checkmark  & \checkmark & 78.42 & 775.09 \\
    \toprule[0.1em]
    MobileNetV2* & $1\times1$ conv   & 32        &            &             &      & 68.76 & 3.23 \\
                 & Image Pooling     & 32        &            &             & \checkmark     & 71.16 & 3.23 \\
                 &                   & 16        &            &             &      & 69.94 & 5.50 \\
                 &                   & 16        &            &             & \checkmark     & {\bf 75.32} & {\bf 5.50} \\
                 &                   &  8 &            &             &      & 70.07 & 17.70 \\
                 &                   &  8 & \checkmark &             &      & 71.40 & 152.60 \\
                 &                   &  8 & \checkmark & \checkmark  &      & 71.54 & 305.18 \\
                 &                   &  8 & \checkmark & \checkmark  & \checkmark & 77.33 & 305.18\\
                 &                   & 16             & 0.5        &             &      & 56.45 & 1.46 \\
                 &                   & 16            & 0.5        &             & \checkmark     & 61.42 & 1.46 \\
                 &                   & 16             & 0.75        &             &      & 67.29 & 3.15 \\
                 &                   & 16            & 0.75        &             & \checkmark     & 71.96 & 3.15 \\
    \toprule[0.1em]
    MobileNetV2* & $1\times1$ conv  & 32      &            &             &      & 62.33 & 3.11 \\
                 &                  & 16      &            &             &      & 62.96 & 5.04 \\
    \toprule[0.1em]
    \toprule[0.1em]
    ResNet-v1-101 & ASPP(6, 12, 18) &  16        &            &             &      & 77.21 & 162.04 \\
                  & $1\times1$ conv &  16        &            &             &  \checkmark    & 80.49 & 162.04 \\
                  & Image Pooling   &  8        &            &              &      & 78.51 & 552.35 \\
                  &                 &  8 & \checkmark &  &  & 79.45 & 4870.74 \\
                  &                 &  8 & \checkmark & \checkmark &  & 79.77 & 9741.18 \\
                  &                 &  8 & \checkmark & \checkmark & \checkmark & 82.70 & 9741.18 \\
    \bottomrule[0.1em]
  \end{tabular}
}
  \caption{\small \mbox{MobileNet} + DeepLabv3 inference strategy on the \mbox{PASCAL} VOC 2012 \textit{validation} set.
  \mbox{\bf MobileNetV2*}: Second last feature map is used for \mbox{DeepLabv3} heads, which includes (1) ASPP with three different rates, (2) $1\times 1$ convolution, and (3) image-level feature.
  {\bf ASPP}: Atrous spatial pyramid pooling.
  {\bf OS}: \emph{output\_stride}.
  {\bf MS}: Multi-scale inputs during test (if specified by a number, the input is resized by that value).
  {\bf Flip}: Adding left-right flipped inputs.
  {\bf COCO}: Model pretrained on MS-COCO. The potential candidate for on-device applications is shown in bold face.
    Note that \mbox{PASCAL} images have dimension around $512\times 512$.}
  \label{tab:mobilenet_val}
\end{table*}}

\subsection{Ablation study}
{\noindent \bf Inverted residual connections.}
The importance of residual connection has been studied extensively~\cite{ResNet,ResNext2016,InceptionV4}.
The new result reported in this paper is that the shortcut connecting bottleneck perform better than shortcuts connecting the expanded layers (see Figure~\ref{fig:residual-impact} for comparison).

{\noindent \bf Importance of linear bottlenecks.} The linear
bottleneck models are strictly less powerful than models with
non-linearities, because the activations can always operate in linear
regime with appropriate changes to biases and scaling. However 
our experiments shown in Figure~\ref{fig:linearity-impact} indicate that 
linear bottlenecks improve performance, providing support that non-linearity destroys information in low-dimensional space. 
\omt{
{\em Impact of expansion rate}
Figure~\ref{fig:effect_of_expansion_rate} shows the impact of different expansion rates at different points of expansion curve. Interestingly, the the expansion factor within a certain range of 5 to 8 seems to mostly match the performance curve, with minor variations which suggests that there might be a trade-off in the model expressivity described by (roughly) number of multiply adds. 
}

\omt{
\subsection{Layer removal experiments and robustness to noise}
Similarly to \cite{ResNetShallow}, we observe that it is possible to remove a trained layer from a network with only a small degradation of the network  performance.
We stress here that this removal does not include a consequent retraining of the network, rather it is a mechanical transformation where we remove 1 or more layers and re-wire the connections.
In our experiments, we removed layers 7 through 12 and made a surprising finding that even though the transformation operation goes through a high-dimensional space\footnote{This is in contrast with traditional residual blocks, where the transformation is low-dimensional.}, the resulting layer operation has very gradual impact on performance.
Interestingly, while both linear and non-linear bottleneck layer exhibit this resilience, linearity seem to \oldtext{significantly improve it}.
In our experiments, this property is only observed for residual networks, but interestingly enough the non-linear bottleneck version of the network seems to  exhibit degraded performance. 
}

\label{sec:layer-removal-experiments}

\section{Conclusions and future work}
We described a very simple network architecture that allowed us to build a family of highly efficient mobile models. Our basic building unit, has several properties that make it particularly suitable for mobile applications. It allows very memory-efficient inference and relies utilize standard operations present in all neural frameworks.

For the \mbox{ImageNet} dataset, our architecture improves the state of the art for wide range of performance points. 

For object detection task, our network outperforms state-of-art realtime detectors on COCO dataset both in terms of accuracy and model complexity. Notably, our architecture combined with the \mbox{SSDLite} detection module is $20\times$ less computation and $10\times$ less parameters than \mbox{YOLOv2}.

On the theoretical side: the proposed convolutional block has a unique property that allows to separate  the network expressiveness (encoded by expansion layers) from its capacity (encoded by bottleneck inputs). Exploring this is an important direction for future research.

\paragraph{Acknowledgments} We would like to thank Matt Streeter and Sergey Ioffe for their helpful feedback and discussion.

\bibliographystyle{unsrt}
\bibliography{bibliography}

\appendix
\def\vecs{{s_1,\dots,s_m}}
\def\interior{\operatorname{interior}}

\section{Bottleneck transformation}
\label{sec:appendix_bottle}

In this section we study the properties of an operator $A\relu (B x)$, where $x\in \real^n$ represents an $n$-channel pixel, $B$ is an $m \times n$ matrix and $A$ is an $n\times m$ matrix.
We argue that if $m \le n$, transformations of this form can only exploit non-linearity at the cost of losing information.
In contrast, if $n \ll m$, such transforms can be highly non-linear but still invertible with high probability (for the initial random weights).

First we show that $\relu$ is an identity transformation for any point that lies in the interior of its image.
\begin{lemma}
Let $S(X) = \{\relu (x) | x \in X\} $. If a volume of $S(X)$ is non-zero, then  $\interior S(X) \subseteq X$.
\end{lemma}
\begin{proof}
Let $S' = \interior\relu (S)$. First we note that if $x \in S'$, then $x_i > 0$ for
all $i$. Indeed, image of $\relu$ does not contain points with negative coordinates, and points with zero-valued coordinates can not be interior points. Therefore for each $x \in S'$, $x = \relu (x)$
as desired.
\end{proof}

It follows that for an arbitrary composition of interleaved linear transformation and $\relu$ operators, if it preserves non-zero volume, that part of the input space $X$ that is preserved over such a composition is a linear transformation, and thus is likely to be a minor contributor to the power of deep networks.
However, this is a fairly weak statement. Indeed, if the input manifold can be embedded into $(n-1)$-dimensional manifold (out of $n$ dimensions total), the lemma is trivially true, since the starting volume is $0$. In what follows we show that when the dimensionality of input manifold is significantly lower we can ensure that there will be no information loss.

Since the $\relu(x)$ nonlinearity is a surjective function mapping the entire ray $x\le 0$ to $0$, using this nonlinearity in a neural network can result in information loss.
Once $\relu$ collapses a subset of the input manifold to a smaller-dimensional output, the following network layers can no longer distinguish between collapsed input samples.
In the following, we show that bottlenecks with sufficiently large expansion layers are resistant to information loss caused by the presence of $\relu$ activation functions.
\begin{figure*}[t]
  \centering

  \begin{subfigure}{.47\textwidth}
    \centering
    \includegraphics[width=.95\linewidth,keepaspectratio=true]
    {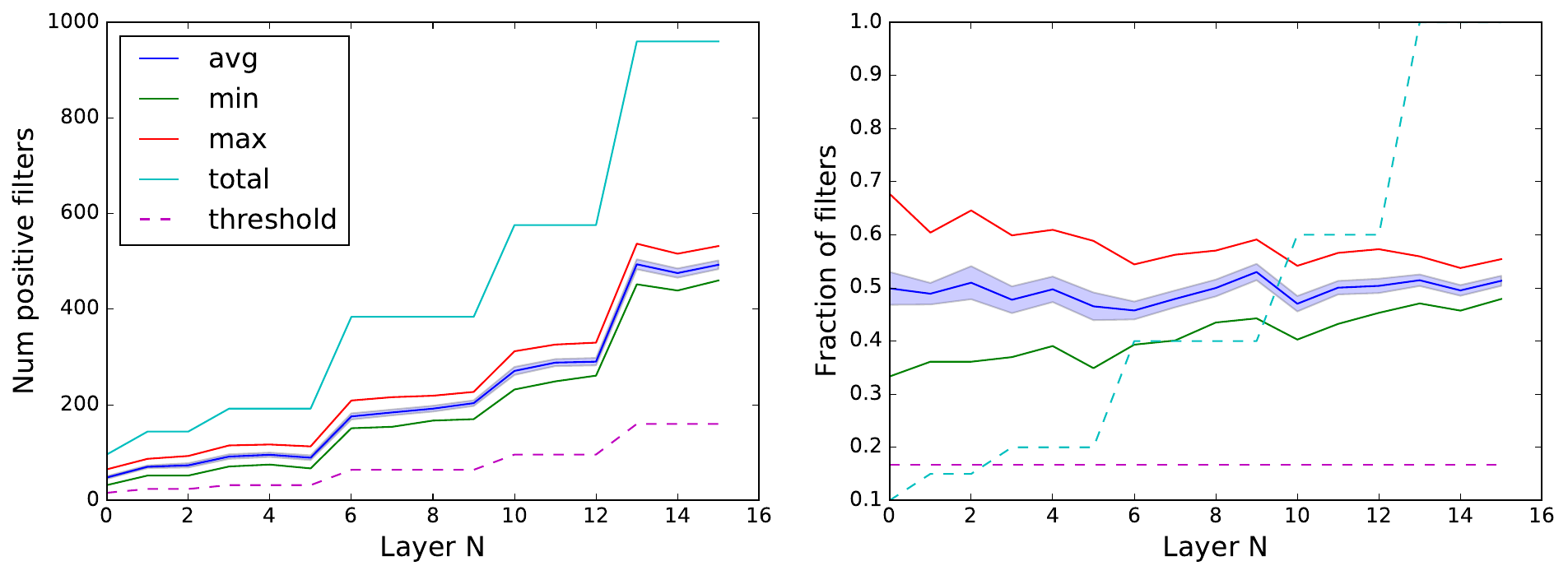}
    \caption {At step 0}
    \label{fig:activation_patterns_random}
  \end{subfigure}
  \begin{subfigure}{.47\textwidth}
    \includegraphics[width=.88\linewidth,keepaspectratio=true]
    {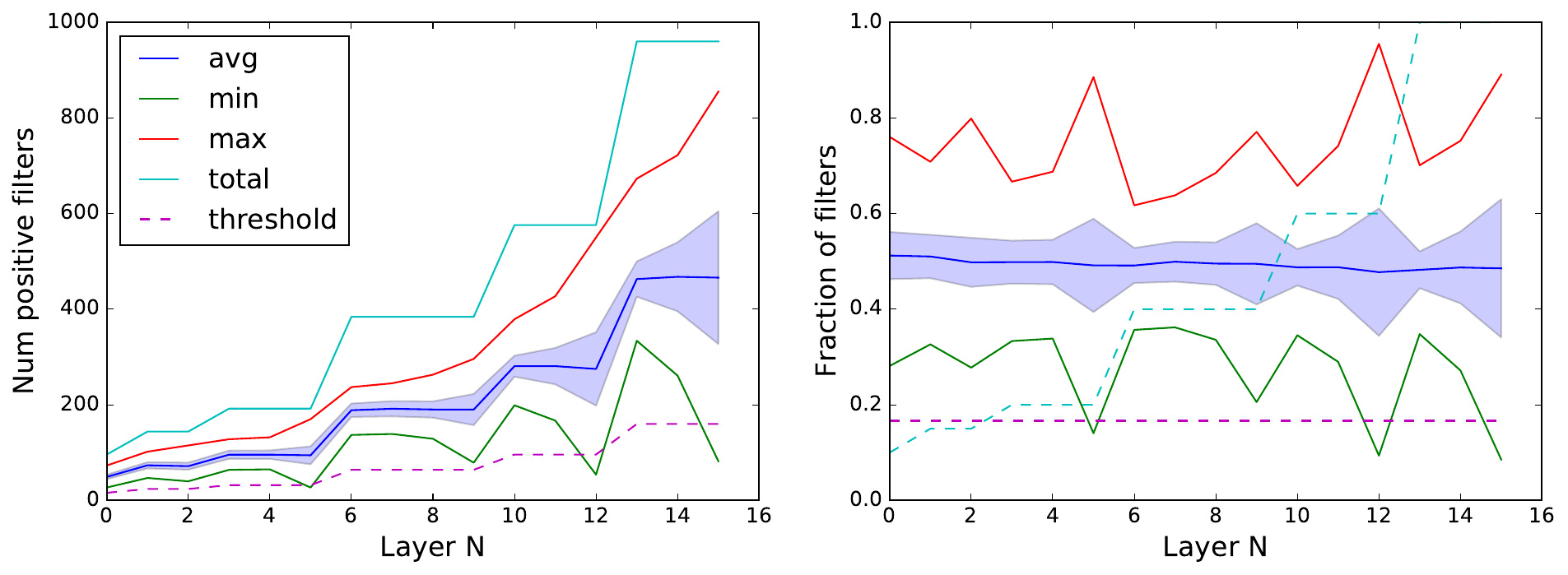}
    \caption {Fully trained}
    \label{fig:activation_patterns_trained}
  \end{subfigure}
  \caption{Distribution of activation patterns. The $x$-axis is the layer index, and we show minimum/maximum/average number of  positive channels after each convolution with $\relu$.
  $y$-axis is either absolute or relative number of channels. The ``threshold'' line indicates the $\relu$
  invertibility threshold - that is the number of positive dimensions is higher than the input space. In our case this is $1/6$ fraction of the channels. Note how at the
  beginning of the training on Figure~\ref{fig:activation_patterns_random} the distribution is much more tightly
  concentrated around the mean. After the training has finished (Figure~\ref{fig:activation_patterns_trained}), the average hasn't changed but the standard deviation grew dramatically. \beveco }
  \label{fig:activation_patterns}
\end{figure*}

\def\E{\mathbb{E}}
\begin{lemma}[Invertibility of ReLU]
\label{thm:invertibility-of-relu}
Consider an operator $\relu (B x)$, where $B$ is an $m \times n$ matrix and $x \in \real^n$. Let $y_0 = \relu (B x_0)$ for some $x_0\in \real^n$, then equation $y_0=\relu (B x)$ has a unique solution with respect to $x$ if and only if $y_0$ has at least $n$ non-zero values and there are $n$ linearly independent rows of $B$ that correspond to non-zero coordinates of $y_0$.
\end{lemma}
\begin{proof}
    Denote the set of non-zero coordinates of $y_0$ as $T$ and let $y_T$ and $B_T$ be restrictions of $y$ and $B$ to the subspace defined by $T$.
    If $|T| < n$, we have $y_T = B_T x_0$ where $B_T$ is under-determined with at least one solution $x_0$, thus there are infinitely many solutions. 
    Now consider the case of $|T| \ge n$ and let the rank of $B_T$ be $n$. Suppose there is an additional solution $x_1 \ne x_0$ such that $y_0=\relu (B x_1)$, then we have $y_T = B_T x_0 = B_T x_1$, which cannot be satisfied unless $x_0 = x_1$.
\end{proof}

One of the corollaries of this lemma says that if $m \gg n$, we only need a small fraction of values of $Bx$ to be positive for $\relu (B x)$ to be invertible.

The constraints of the lemma \ref{thm:invertibility-of-relu} can be empirically validated for real networks and real inputs and hence we can be assured that information is indeed preserved. We further show that with respect to initialization, we can be sure that these constraints are satisfied with high probability. Note that for random initialization the conditions of
lemma \ref{thm:invertibility-of-relu} are satisfied due
to initialization symmetries. However even for trained graphs these constraints can be empirically validated by running
the network over valid inputs and verifying that all or most inputs are above the threshold. On Figure \ref{fig:activation_patterns} we show how this distribution looks for  different MobileNetV2 layers. At step $0$ the activation patterns concentrate around having half of the positive channel (as predicted by initialization symmetries). For fully trained network, while the standard deviation grew significantly, all but the two layers are still above the invertibility thresholds. We believe further study of this is warranted and might lead to helpful insights on network design.

\begin{theorem}
	Let $S$ be a compact $n$-dimensional submanifold of $\real^n$.
	Consider a family of functions $f_B(x)=\relu (B x)$ from $\real^n$ to $\real^m$ parameterized by $m\times n$ matrices $B\in \mathcal{B}$.
	Let $p(B)$ be a probability density on the space of all matrices $\mathcal{B}$ that satisfies:
	\begin{itemize}
		\item $P(Z)=0$ for any measure-zero subset $Z\subset \mathcal{B}$;
		\item (a symmetry condition) $p(D B) = p(B)$ for any $B \in \mathcal{B}$ and any $m \times m$ diagonal matrix $D$ with all diagonal elements being either $+1$ or $-1$.
	\end{itemize}
	Then, the average $n$-volume of the subset of $S$ that is collapsed by $f_B$ to a lower-dimensional manifold is 
	$$
		V - \frac{N_{m,n} V}{2^m},
	$$
	where $V = \vol S$ and 
	$$
		N_{m,n} \equiv \sum_{k=0}^{m-n} \binom{m}{k}.
	$$
\end{theorem}
\begin{proof}
	For any $\sigma=(s_1,\dots,s_m)$ with $s_k \in \{-1,+1\}$, let $Q_\sigma=\{x\in \real^m| x_i s_i > 0\}$ be a corresponding quadrant in $\real^m$.
	For any $n$-dimensional submanifold $\Gamma \subset \real^m$, $\relu$ acts as a bijection on $\Gamma \cap Q_{\sigma}$ if $\sigma$ has at least $n$ positive values\footnote{unless at least one of the positive coordinates for all $x\in \Gamma \cap Q_{\sigma}$ is fixed, which would not be the case for almost all $B$ and $\Gamma = BS$} and contracts $\Gamma \cap Q_{\sigma}$ otherwise.
	Also notice that the intersection of $B S$ with $\real^m \backslash (\cup_\sigma Q_\sigma)$ is almost surely $(n-1)$-dimensional.
	The average $n$-volume of $S$ that is not collapsed by applying $\relu$ to $B S$ is therefore given by:
	\begin{equation}
    	\label{eq:presum}
		\sum_{\sigma\in \Sigma_{n}} \E_{B} [V_\sigma(B)],
	\end{equation}
	where $\Sigma_{n}=\left\{(s_1,\dots,s_m)|\sum_k \theta(s_k) \ge n\right\}$, $\theta$ is a step function and $V_\sigma(B)$ is a volume of the largest subset of $S$ that is mapped by $B$ to $Q_\sigma$.
	Now let us calculate $\E_B [V_\sigma(B)]$.
	Recalling that $p(D B)=p(B)$ for any $D=\diag(s_1,\dots,s_m)$ with $s_k \in \{-1,+1\}$, this average can be rewritten as $\E_B \E_D [V_\sigma(DB)]$.
	Noticing that the subset of $S$ mapped by $D B$ to $Q_\sigma$ is also mapped by $B$ to $D^{-1} Q_\sigma$, we immediately obtain $\sum_{\sigma'} V_\sigma[\diag(\sigma') B] = \sum_{\sigma'} V_{\sigma'}[B] = \vol S$ and therefore $\E_{B} [V_\sigma(B)] = 2^{-m} \, \vol S$.
	Substituting this and $|\Sigma_n|=\sum_{k=0}^{m-n} \binom{m}{k}$ into Eq.~\ref{eq:presum} concludes the proof.
\end{proof}

Notice that for sufficiently large expansion layers with $m \gg n$, the fraction of collapsed space $N_{m,n}/2^m$ can be bounded by:
$$
  \frac{N_{m,n}}{2^m} \ge 1 - \frac{m^{n+1}}{2^m n!} \ge
  1 - 2^{(n+1)\log m - m } \ge 1 - 2^{-m/2}
$$
and therefore $\relu(B x)$ performs a nonlinear transformation while preserving information with high probability.  

We discussed how bottlenecks can prevent manifold collapse, but increasing the size of the bottleneck expansion may also make it possible for the network to represent more complex functions. Following the main results of ~\cite{Montufar2014}, one can show, for example, that for any integer $L\ge 1$ and $p>1$ there exist a network of $L$ $\relu$ layers, each containing $n$ neurons and a bottleneck expansion of size $p n$ such that it maps $p^{nL}$ input volumes (linearly isomorphic to $[0,1]^{n}$) to the same output region $[0,1]^n$.
Any complex possibly nonlinear function attached to the network output would thus effectively compute function values for $p^{nL}$ input linear regions.

\section{Semantic segmentation visualization results}
\label{sec:appendix_segmentation}
\begin{figure*}[!t]
  \centering
  \scalebox{0.9}{
  \begin{tabular}{c}
    \includegraphics[width=0.96\linewidth]{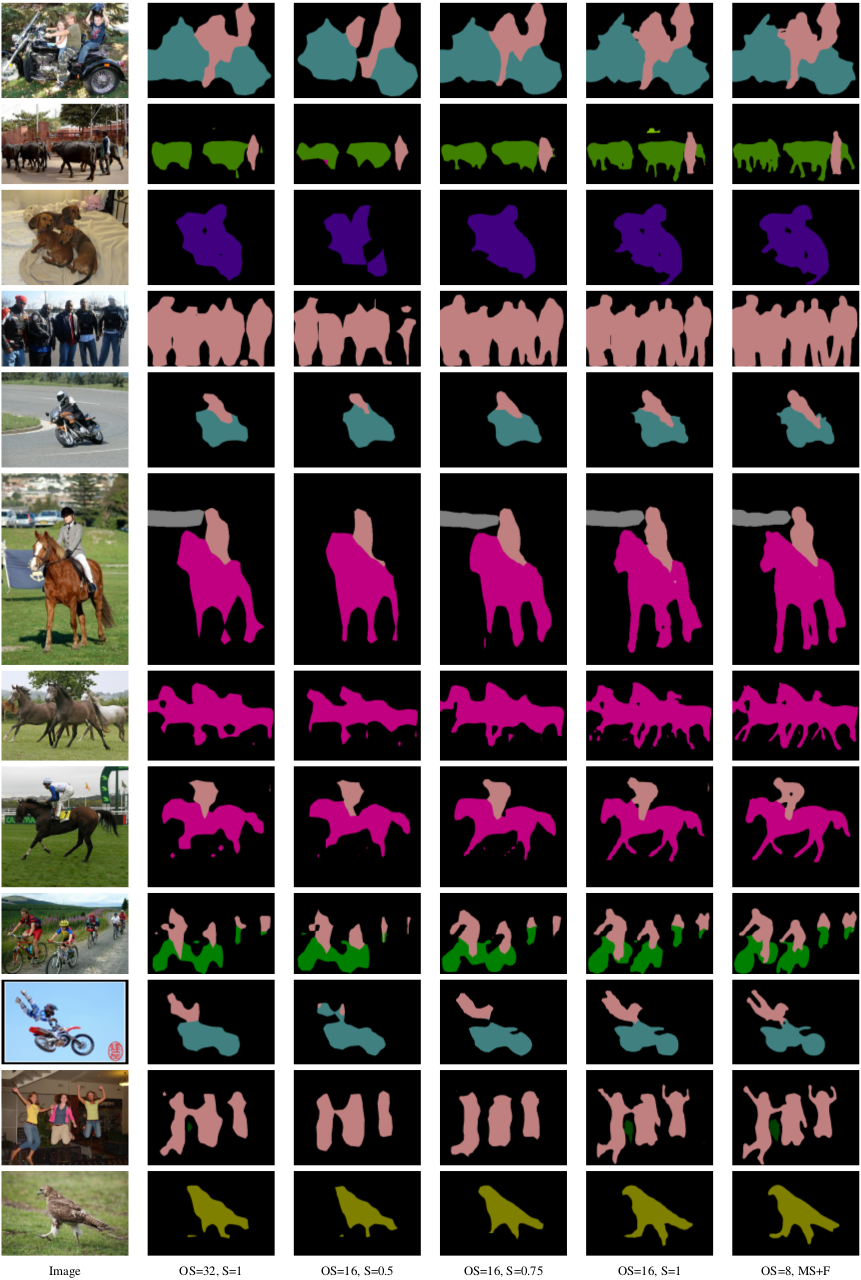} \\
  \end{tabular}
  }
  \caption{\mbox{MobileNetv2} semantic segmentation visualization results on \mbox{PASCAL} VOC 2012 {\it val} set. \textbf{OS}: $\emph{output\_stride}$. \textbf{S}: single scale input. \textbf{MS+F}: Multi-scale inputs with scales = $\{0.5, 0.75, 1, 1.25, 1.5, 1.75\}$ and left-right flipped inputs. Employing $\emph{output\_stride}=16$ and single input scale = 1 attains a good trade-off between FLOPS and accuracy.}
  \label{fig:vis_mobilenetv2}
\end{figure*}
 
\end{document}